\documentclass{article}

 \usepackage[preprint]{neurips_2026}
 \usepackage[utf8]{inputenc} 
\usepackage[T1]{fontenc}    
\usepackage{hyperref}       
\usepackage{url}            
\usepackage{amsfonts}       
\usepackage{nicefrac}       

\usepackage{amsmath}
\usepackage{amssymb}
\usepackage{amsthm}    
\usepackage{dsfont}
\usepackage{booktabs}
\usepackage{dsfont}
\usepackage{graphicx}
\usepackage{makecell}
\usepackage{multirow}
\usepackage{pifont}
\usepackage[group-separator={,}]{siunitx}
\usepackage{tabularx}

\usepackage{xcolor}
\usepackage{caption}
\usepackage{soul}
\usepackage{changepage,threeparttable}
\usepackage{array}
\usepackage{colortbl}
\usepackage{algpseudocode}
\usepackage{algorithm}
\usepackage{mathtools}
\usepackage{xfrac}

\usepackage{catchfile}
\usepackage{longtable}
\usepackage{geometry}
\usepackage{enumitem}
\usepackage{tabu}
\usepackage{supertabular}
\usepackage{multicol}
\usepackage{algpseudocode}
\usepackage{algorithm}
\usepackage{subfig}
\usepackage{wrapfig}
\usepackage{subcaption}

\newcolumntype{Y}{>{\centering\arraybackslash}X}

\definecolor{color_1}{RGB}{255,0,128}
\definecolor{color_2}{RGB}{0,128,128}
\definecolor{color_3}{RGB}{0,128,0}
\definecolor{color_4}{RGB}{128,0,0}
\definecolor{color_5}{RGB}{128,0,128}
\definecolor{cadetgrey}{RGB}{0.57, 0.64, 0.69}
\definecolor{our_color}{RGB}{0,65,145}

\newcommand{\KL}{\mathrm{KL}}
\DeclareMathOperator*{\argmax}{arg\,max}

\newtheorem{proposition}{Proposition}

\newtheorem{lemma}{Lemma}
\newtheorem{remark}{Remark}

\usepackage{amssymb}
\usepackage{pifont}
\newcommand{\cmark}{\ding{51}}%
\newcommand{\xmark}{\ding{55}}%

\newcommand{\OC}[1]{\textcolor{our_color}{#1}}

\title{Drifting Field Policy: A One-Step Generative Policy via Wasserstein Gradient Flow}

\author{
Juil Koo $\quad$
Mingue Park $\quad$
Jiwon Choi $\quad$
Yunhong Min $\quad$
Minhyuk Sung \\[0.2em]
KAIST \\
{\tt\small \texttt{\{63days, kicikicik\}@kaist.ac.kr, jwchoi1529@gmail.com,}}\\
{\tt\small \texttt{\{dbsghd363, mhsung\}@kaist.ac.kr}}
}

\begin{document}
\maketitle

\begin{abstract}
We propose \textbf{Drifting Field Policy (DFP)}, a non-ODE one-step generative policy built on the \emph{drifting model} paradigm~\cite{Deng2026:Drifting}. We frame the policy update as a reverse-KL Wasserstein-2 gradient flow toward a soft target policy, so that each DFP update corresponds to a gradient step in probability space. By construction, this gradient is decomposed into an ascent toward higher action-value regions and a score matching with the anchor policy as a trust region. We further derive a simple, tractable surrogate of the otherwise intractable update loss, akin to behavior cloning on top-$K$ critic-selected actions. We find empirically that this mechanism uniquely benefits the drifting backbone owing to its non-ODE parameterization. With one-step inference, DFP achieves state-of-the-art performance on several manipulation tasks across Robomimic and OGBench, outperforming ODE-based policies.
\end{abstract}
\section{Introduction}
\label{sec:intro}
Offline-to-online reinforcement learning (RL) has emerged as a practical paradigm for continuous control~\cite{Nair2020:AWAC, kostrikov2022:IQL, nakamoto2023:CalQL, ball2023:RLPD}, where an RL agent is first pretrained on static demonstrations and then refined through online interaction. To faithfully capture the multimodal action distributions of real-world demonstrations beyond unimodal Gaussians, the field has increasingly turned to generative policies, with ODE-based backbones such as diffusion and flow policies~\cite{wang2022:DiffusionQL, hansen2023:IDQL, ding2024:QVPO, Park2025:FQL, Ding2024:consistencypolicy} proving particularly effective. To further meet the low-latency demands of deployment, recent work has converged on one-step inference, realized by one-step variants of these ODE backbones~\cite{frans2024:shortcut, geng2025:mean, Song:2023CM, Kim:2023CTM} that amortize ODE trajectory integration into a single forward pass. 


Despite the success of generative policies in offline behavior cloning~\cite{Black2025:Pi0, Janner2022:Diffuser, Prasad2024:ConsistencyPolicy, Chi2025:DiffusionPolicy}, RL finetuning exposes a structural burden for ODE-based parameterizations: a reward signal defined at the action must propagate back through the entire ODE trajectory, posing a non-trivial output-to-trajectory credit assignment problem~\cite{Ren2025:DPPO, Li2026:QAM, ding2024:QVPO, Mcallister2026:FPO, Black2024:DDPO, Kim2025:RBF}. Crucially, this burden persists even in one-step variants, whose training objective is still defined along the ODE path. This motivates us to seek a one-step policy parameterization that bypasses trajectory integration entirely, so that output-level reward signals can act directly at the action level.

We propose \textbf{Drifting Field Policy (DFP)}, a non-ODE one-step generative policy built on \emph{drifting models}~\cite{Deng2026:Drifting}. The policy is a single-pass pushforward map $\pi_\theta(\cdot|s) = [f_\theta(\cdot, s)]_\# p_\epsilon$ from a prior $p_\epsilon$ to the action space, with no time variable. Drifting models are trained via a \emph{drifting field} update that combines attraction toward a target distribution with repulsion from the current model distribution, driving $\pi_\theta$ toward the target. In RL fine-tuning, with the soft policy improvement target $\pi^+ \propto \pi_{\mathrm{old}}\exp(Q/\alpha)$~\cite{Levine2018:RLasInference, Haarnoja2018:SAC}, we frame the corresponding drifting field update as a reverse-KL Wasserstein-2 gradient flow~\cite{Cao2026:GFD} that minimizes $\KL(\pi_\theta \| \pi^+)$: the ideal drift field follows the steepest-descent direction toward $\pi^+$ on probability space. We further show that this gradient is structurally decomposed into a $\nabla_a Q$ ascent direction and a score matching with the anchor policy $\pi_{\mathrm{old}}$ as a trust region.

However, $\pi^+$ in this KL is intractable due to the normalizing constant. We therefore propose a simple yet effective tractable surrogate that replaces $\pi^+$ with the top-$K$ critic-selected actions as positive targets. The training objective under this approximation is akin to behavior cloning on these $K$ self-generated candidates, thus easy to implement with no architectural change. We prove this surrogate has bounded approximation error to the ideal update. Since DFP performs gradient descent directly on \emph{probability space}, each step shifts the policy distribution directly at the action output, in contrast to diffusion and flow policies~\cite{wang2022:DiffusionQL, hansen2023:IDQL, ding2024:QVPO,Park2025:FQL, Zhan2026:MVP} that update the velocity prediction defining their ODE, spreading each signal across the ODE trajectory. Ablations confirm this distinction: the same top-$K$ supervision yields only marginal gains on a MeanFlow backbone~\cite{geng2025:mean}.



We summarize our contributions as follows:
(i) We introduce the first application of drifting models to RL fine-tuning, framing the ideal drifting-field update as a Wasserstein-2 gradient flow direction on probability space — the steepest-descent direction toward the soft policy improvement target $\pi^+$ (Sec.~\ref{sec:method}).
(ii) We derive a tractable top-$K$ surrogate of the otherwise intractable target $\pi^+$ with a bounded approximation error to the ideal policy improvement loss (Proposition~\ref{prop:topk_limit}), and show that this top-$K$ supervision is particularly effective on the drifting backbone, in contrast to ODE-based one-step backbones whose velocity-level updates can spread each positive's signal across the ODE trajectory (Sec.~\ref{sec:why_drifting}).
(iii) On 12 tasks across Robomimic~\cite{robomimic2021} and OGBench~\cite{park2024:OGbench} benchmarks, DFP achieves state-of-the-art performance on 9 of 12 tasks and second-best on the remaining 3, outperforming prior ODE-based generative policies, such as QC-FQL~\cite{Li2025:QC} and MVP~\cite{Zhan2026:MVP}, by a large margin on average (Sec.~\ref{sec:results}).
    
\section{Preliminaries}
\label{sec:preliminaries}

In this section, we briefly review the three building blocks of our method: the offline-to-online RL setting and the off-policy actor-critic framework (Sec.~\ref{sec:offline-to-online-rl}), drifting models as a one-step generative paradigm (Sec.~\ref{sec:prelim_drifting}), and the Wasserstein gradient flow interpretation of the drifting field (Sec.~\ref{sec:prelim_wgf}).

\subsection{Offline-to-Online Reinforcement Learning}
\label{sec:offline-to-online-rl}
We consider a Markov Decision Process $\mathcal{M} = (\mathcal{S}, \mathcal{A}, P, r, \gamma)$ with state space $\mathcal{S}$, action space $\mathcal{A} \subseteq \mathbb{R}^d$, transition dynamics $P(s' | s, a)$, reward $r(s, a) \in \mathbb{R}$, and discount factor $\gamma \in [0, 1)$. The objective of reinforcement learning is to find a policy $\pi(\cdot | s)$ maximizing the expected discounted return $J(\pi) = \mathbb{E}_\pi\bigl[\sum_{k=0}^\infty \gamma^k r(s_k, a_k)\bigr]$.

A common training strategy in offline-to-online RL is to first pretrain a policy on a static offline dataset $\mathcal{D}_{\mathrm{offline}}$ (offline stage) and then fine-tune it with a limited budget of online interactions (online stage). 
We denote by $\mathcal{D}$ the replay buffer that accumulates both $\mathcal{D}_{\mathrm{offline}}$ and the online transitions, and by $\pi_\beta$ the (potentially unknown) behavior distribution that generates the data in $\mathcal{D}$.

We adopt the standard off-policy actor-critic framework~\cite{Konda1999:AC}, in which the critic $Q_\phi$ estimates the expected discounted return for each state-action pair via a temporal difference loss, and the actor $\pi_\theta$ is updated to maximize $Q_\phi$ under a behavioral constraint~\cite{Wu2019:BRAC, FujimotoandGu2021:TD3BC, tarasov2023:ReBRAC} that keeps it close to $\pi_\beta$. The two are jointly trained by minimizing:
\begin{align}
\label{eq:critic_td}
\mathcal{L}_Q(\phi) &= \mathbb{E}_{(s, a, r, s') \sim \mathcal{D}}\bigl[(Q_\phi(s, a) - r - \gamma\, Q_{\bar\phi}(s', a'))^2\bigr], \quad a' \sim \pi_\theta(\cdot | s'), \\
\label{eq:actor_loss}
\mathcal{L}_\pi(\theta) &= -\mathbb{E}_{s \sim \mathcal{D},\, a \sim \pi_\theta(\cdot|s)}[Q_\phi(s, a)], \quad \text{s.t.}\ D(\pi_\theta(\cdot|s),\, \pi_\beta(\cdot|s)) \leq \varepsilon,
\end{align}
where $Q_{\bar\phi}$ is a target network~\cite{Lillicrap2016:DDPG, Mnih2015:DQN} and $D$ is a divergence between $\pi_\theta$ and $\pi_\beta$. 
In Sec.~\ref{sec:method}, we instantiate this constrained actor objective with a novel one-step generative policy built on drifting models.

\subsection{Drifting Models}
\label{sec:prelim_drifting}

Drifting models~\cite{Deng2026:Drifting} are a recent paradigm for one-step generative modeling. Let $p := p_{\mathrm{data}}$ denote the data distribution on $\mathbb{R}^d$ and $p_\epsilon := \mathcal{N}(0, I)$ a prior on $\mathbb{R}^m$. Instead of describing transport from $p_\epsilon$ to $p$ through a stochastic process or its ODE counterpart at inference as in diffusion and flow models~\cite{Ho:2020DDPM, Song:2021DDIM, Song:2021ScoreSDE, Lipman:2023FM}, drifting models shift the dynamics to training: they directly parameterize single-pass pushforward map $f_\theta: \mathbb{R}^m \to \mathbb{R}^d$ trained so that the model distribution $q := [f_\theta]_\# p_\epsilon$ matches $p$.

The core of drifting models is the \emph{drifting field} $\mathbf{V}_{p, q} = \mathbf{V}^+_p - \mathbf{V}^-_q$, a vector field that comprises an attraction term $\mathbf{V}^+_p$ and a repulsion term $\mathbf{V}^-_q$, both constructed via kernel mean shift:
\begin{align}
\mathbf{V}^+_p(x) = \frac{\mathbb{E}_{y^+ \sim p}\bigl[k(x, y^+)(y^+ - x)\bigr]}{\mathbb{E}_{y^+ \sim p}[k(x, y^+)]}, \qquad
\mathbf{V}^-_q(x) = \frac{\mathbb{E}_{y^- \sim q}\bigl[k(x, y^-)(y^- - x)\bigr]}{\mathbb{E}_{y^- \sim q}[k(x, y^-)]},
\label{eq:attraction_repulsion}
\end{align}
where $k: \mathbb{R}^d \times \mathbb{R}^d \to \mathbb{R}_{>0}$ is a similarity kernel with bandwidth $h$ (e.g., the Gaussian kernel $k(x, y) = \exp(-\|x - y\|^2 / (2h^2))$), and $y^+ \sim p$ and $y^- \sim q$ denote positive and negative samples drawn from the two distributions. Intuitively, the attraction term pulls generated samples toward nearby data, while the repulsion term pushes them away from one another to prevent mode collapse.

By construction, the drifting field is anti-symmetric, $\mathbf{V}_{p,q}(x) = -\mathbf{V}_{q,p}(x)$, which implies $q=p \Rightarrow \mathbf{V}_{p,q} \equiv 0$. Under mild kernel regularity conditions, the converse also holds in the sense that $\mathbf{V}_{p,q} \approx 0 \Rightarrow q \approx p$~\cite{Deng2026:Drifting}. Drifting models are trained to satisfy $\mathbf{V}_{p, q} = 0$ via fixed-point regression with the stop-gradient operator ``$\mathrm{sg}$'' on the drifted target:
\begin{equation}
\mathcal{L}_{\mathrm{drift}}(\theta;\, p, q) \;=\; \mathbb{E}_{\epsilon \sim p_\epsilon}\bigl[\|x - \mathrm{sg}(x + \mathbf{V}_{p, q}(x))\|^2 \bigr], \text{ where } x = f_\theta(\epsilon).
\label{eq:Ldrift_general}
\end{equation}
The drift loss is parameterized by the positive (target) distribution $p$ and the negative (source) distribution $q$, and the same form applies to any choice of $(p, q)$ beyond unconditional generative modeling. In Sec.~\ref{sec:method}, we exploit this flexibility by plugging in different positive distributions for policy improvement and behavior cloning.



\subsection{Drifting Field as a Wasserstein Gradient Flow}
\label{sec:prelim_wgf}

A recent line of work~\cite{Cao2026:GFD, He2026:Sinkhorn} reveals that the drifting field $\mathbf{V}_{p, q}$ of Sec.~\ref{sec:prelim_drifting} is precisely the particle velocity of a Wasserstein-2 gradient flow (WGF) under KDE-smoothed densities. We recall this identification, which underpins the policy-space treatment in Sec.~\ref{sec:method}.


\paragraph{Wasserstein-2 gradient flow.} 
Equip the space $\mathcal{P}_2(\mathbb{R}^d)$ of probability measures with finite second moment with the Wasserstein-2 distance $W_2$~\cite{Ambrosio2005:GF, Santambrogio2015:optimal}. An absolutely continuous curve $\{q_t\}_{t \geq 0} \subset \mathcal{P}_2(\mathbb{R}^d)$ is characterized by a time-dependent velocity field $v_t: \mathbb{R}^d \to \mathbb{R}^d$ that transports each particle along $\dot{x}_t = v_t(x_t)$ and induces the marginal evolution $\partial_t q_t + \nabla \cdot (q_t v_t) = 0$ through the continuity equation. For a smooth functional $\mathcal{F}: \mathcal{P}_2(\mathbb{R}^d) \to \mathbb{R}$ with first variation $\frac{\delta \mathcal{F}}{\delta q}$, the $W_2$ gradient flow~\cite{Jordan1998:Variational} of $\mathcal{F}$ is the absolutely continuous curve whose velocity field is the steepest-descent direction $v_t(x) = -\nabla_x \frac{\delta \mathcal{F}}{\delta q_t}(x)$, equivalently characterized by the PDE
\begin{equation}
    \partial_t q_t = \nabla \cdot \!\left( q_t \nabla_x \frac{\delta \mathcal{F}}{\delta q_t} \right),
    \label{eq:wgf_pde}
\end{equation}
the $W_2$-gradient flow interpretation of ``$\dot{q}_t = -\nabla \mathcal{F}(q_t)$'' on $\mathcal{P}_2$. Specializing to $\mathcal{F}(q) = \KL(q \| p)$, the velocity reduces to the score difference (See Appendix~\ref{app:wgf_derivations} for derivation)
\begin{equation}
    v_t(x) = \nabla_x \log p(x) - \nabla_x \log q_t(x),
    \label{eq:wgf_velocity_general}
\end{equation}
which attracts particles toward $p$, repels them from the current $q_t$, and monotonically dissipates $\KL(q_t \| p)$ so that $q_t \to p$ in $W_2$ as $t \to \infty$~\cite{Ambrosio2005:GF}.

\paragraph{Drifting field as KDE-approximated WGF velocity.}
The scores in Eq.~\eqref{eq:wgf_velocity_general} are not available in closed form, since $p$ and $q_t$ are typically only accessible through samples. Replacing each density with its KDE smoothing $\mu_{\mathrm{kde}}(x) := \int k_h(x, y)\, d\mu(y)$ under a Gaussian kernel $k_h(x, y) = \exp(-\|x - y\|^2 / (2h^2))$ yields the score identity~\cite{Cheng1995:Mean}
\begin{equation}
h^2\, \nabla_x \log \mu_{\mathrm{kde}}(x) = \frac{\int k_h(x, y)\,(y - x)\,d\mu(y)}{\int k_h(x, y)\,d\mu(y)}.
\label{eq:kde_grad}
\end{equation}
Substituting Eq.~\eqref{eq:kde_grad} into Eq.~\eqref{eq:wgf_velocity_general}, one can obtain the following identity:
\begin{equation}
h^2\,[\nabla_x \log p_{\mathrm{kde}}(x) - \nabla_x \log q_{\mathrm{kde}}(x)] = \mathbf{V}^+_p(x) - \mathbf{V}^-_q(x) = \mathbf{V}_{p, q}(x).
\label{eq:drift_eq_wgf_general}
\end{equation}
Consequently, the drifting loss of Eq.~\eqref{eq:Ldrift_general} is a parametric KDE-WGF descent of $\KL(q \| p)$, a viewpoint we leverage in Sec.~\ref{sec:method} by specializing $(p, q)$ to the policy learning.


\section{Drifting Field Policy}
\label{sec:method}

We propose \textbf{Drifting Field Policy (DFP)}, a novel one-step generative policy method for RL finetuning based on drifting models~\cite{Deng2026:Drifting}. The policy is a single-pass map $f_\theta: \mathbb{R}^k \times \mathcal{S} \to \mathcal{A}$ that induces a state-conditional pushforward distribution on the action space,
\begin{equation}
\pi_\theta(\cdot | s) \;:=\; [f_\theta(\cdot, s)]_\# p_\epsilon.
\label{eq:pi_theta_def}
\end{equation}
DFP trains $\pi_\theta$ with the drifting training loss $\mathcal{L}_{\mathrm{drift}}(\theta;\, p, q)$ of Sec.~\ref{sec:prelim_drifting}, identifying the model distribution $q$ with $\pi_\theta$ and instantiating the positive distribution $p$ at two complementary targets: the $Q$-maximizing target $\pi^+$ for policy improvement and the behavior distribution $\pi_\beta$ for data anchoring. Sec.~\ref{sec:wgf_theory} derives the training objective with its Wasserstein gradient flow interpretation and tractable top-$K$ surrogate; Sec.~\ref{sec:why_drifting} contrasts with diffusion and flow policies.

\subsection{Policy Improvement via Wasserstein Gradient Flow}
\label{sec:wgf_theory}

Let $\pi_{\mathrm{old}}$ denote a trust-region anchor policy. At each iteration, the goal is to update the current policy $\pi_\theta$ to maximize $Q$ within a trust region around $\pi_{\mathrm{old}}$, given by the standard KL-regularized objective from optimal control~\cite{Kappen2005:linear, Todorov2006:Linearly} and policy search~\cite{Schulman2015:TRPO, Peters2010:REPS, Abdolmaleki2018:Maximum, Levine2014:Learning, Nair2020:AWAC, Peng2019:AWR}:
\begin{align}
\pi^+(\cdot | s) := \argmax_{\pi}\; \mathbb{E}_{a \sim \pi(\cdot | s)}\bigl[Q_\phi(s, a)\bigr] \;-\; \alpha\, D_{\KL}\bigl(\pi(\cdot | s)\, \big\|\, \pi_{\mathrm{old}}(\cdot | s)\bigr),
\label{eq:pi_plus_objective}
\end{align}
with temperature $\alpha > 0$. This optimal policy $\pi^+$ admits the closed-form solution~\cite{Levine2018:RLasInference}:
\begin{equation}
\pi^+(a|s) = \frac{\pi_{\mathrm{old}}(a|s) \exp(Q_\phi(s,a)/\alpha)}{Z(s)}, \quad Z(s) = \int \pi_{\mathrm{old}}(a'|s) \exp(Q_\phi(s,a')/\alpha)\, da'.
\end{equation}
To realize this update under the drifting model parameterization, we instantiate the drifting loss $\mathcal{L}_{\mathrm{drift}}(\theta;\, p, q)$ of Sec.~\ref{sec:prelim_drifting} with $p = \pi^+$ as the positive target and $q = \pi_\theta$ as the negative source. Letting $\hat{a} := f_\theta(\epsilon, s)$,
\begin{equation}
\mathcal{L}_{\mathrm{PI}}(\theta) = \mathcal{L}_{\mathrm{drift}}(\theta;\, \pi^+, \pi_\theta) \;=\; \mathbb{E}_{s, \epsilon}\bigl[\|\hat{a} - \mathrm{sg}(\hat{a} + \mathbf{V}_{\pi^+, \pi_\theta}(\hat{a} | s))\|^2\bigr],
\label{eq:ideal_pi_loss}
\end{equation}
where $\mathbf{V}_{\pi^+, \pi_\theta} = \mathbf{V}^+_{\pi^+} - \mathbf{V}^-_{\pi_\theta}$ is the drifting field between $\pi^+$ and $\pi_\theta$.

\begin{remark}[$\mathcal{L}_{\mathrm{PI}}$ as $\nabla_a Q$ ascent with score matching regularization]
\label{rem:wgf_pi}
Specializing the result of Sec.~\ref{sec:prelim_wgf} to $(p, q) = (\pi^+, \pi_\theta)$, $\mathbf{V}_{\pi^+, \pi_\theta}$ is the KDE-approximated Wasserstein-2 gradient flow velocity of $\KL(\pi_\theta \| \pi^+)$ on policy space~\cite{Cao2026:GFD},
\begin{equation}
\mathbf{V}_{\pi^+, \pi_\theta}(a | s) \;=\; h^2\bigl[\nabla_a \log \pi^+_{\mathrm{kde}}(a | s) - \nabla_a \log \pi_{\theta, \mathrm{kde}}(a | s)\bigr],
\label{eq:drift_pol_wgf}
\end{equation}
so in the ideal nonparametric continuous-time limit, this field gives a
KL-dissipating update direction toward $\pi^+$. Substituting $\nabla_a \log \pi^+ = \tfrac{1}{\alpha}\nabla_a Q_\phi + \nabla_a \log \pi_{\mathrm{old}}$ into Eq.~\eqref{eq:drift_pol_wgf} and the small-bandwidth limit $\log p_{\mathrm{kde}} \to \log p$ yields the structural decomposition
\begin{equation}
\mathbf{V}_{\pi^+, \pi_\theta}(a | s) \;\simeq\; \underbrace{\frac{h^2}{\alpha}\,\nabla_a Q_\phi(s, a)}_{\nabla_a Q\ \text{ascent}} \;+\; \underbrace{h^2\, \bigl(\nabla_a \log \pi_{\mathrm{old}}(a | s) - \nabla_a \log\pi_\theta(a | s)\bigr)}_{\text{trust region around}\ \pi_{\mathrm{old}}\ \text{via score matching}},
\label{eq:wpo_decomp}
\end{equation}
revealing that the ideal drift field contains an action-space $\nabla_a Q$
ascent component at temperature $\alpha$ regularized by score matching between $\pi_\theta$ and the anchor $\pi_{\mathrm{old}}$, requiring neither the critic Jacobian $\nabla_a Q_\phi$ of DDPG-style actor updates~\cite{Lillicrap2016:DDPG} nor an explicit KL computation. Derivation in Appendix~\ref{app:wgf_derivations}.
\end{remark}

\paragraph{Tractable surrogate.}
While Eq.~\eqref{eq:ideal_pi_loss} provides the desired soft policy update~\cite{Levine2018:RLasInference, Haarnoja2018:SAC} without an explicit critic Jacobian or KL computation, it is itself not directly tractable: $\pi^+ \propto \pi_{\mathrm{old}}\exp(Q_\phi/\alpha)$ has an intractable normalization $Z(s)$ and cannot be sampled directly. Thus, we propose a simple yet effective surrogate loss as follows. A natural starting point is self-normalized importance sampling: drawing $N$ candidate actions $a^{(1)}, \ldots, a^{(N)} \overset{\text{i.i.d.}}{\sim} \pi_{\mathrm{old}}(\cdot | s)$ and weighting them by $w_j \propto \exp(Q_\phi(s, a^{(j)})/\alpha)$ to yield an estimator of expectations under $\pi^+$. We observe, however, that an even simpler scheme — replacing the weights with a uniform hard top-$K$ cutoff, so the top-$K$ candidates by $Q_\phi$ are equally weighted — is empirically robust to $K$ over a wide range and admits a bounded approximation error to $\mathcal{L}_{\mathrm{PI}}$ (Proposition~\ref{prop:topk_limit}). 

Denoting the resulting positive set by
\begin{equation}
P_K(s) \;:=\; \bigl\{a^{(j)} \;:\; j \in \mathrm{argTopK}_{j \in [N]}\, Q_\phi(s, a^{(j)})\bigr\}, \qquad a^{(j)} \overset{\text{i.i.d.}}{\sim} \pi_{\mathrm{old}}(\cdot | s),
\label{eq:PK_def}
\end{equation}
the tractable surrogate is
\begin{equation}
\mathcal{L}_{\text{top-}K}(\theta) = \mathcal{L}_{\mathrm{drift}}(\theta;\, P_K, \pi_\theta) \;=\; \mathbb{E}_{s, \epsilon}\bigl[\|\hat{a} - \mathrm{sg}(\hat{a} + \mathbf{V}_{P_K, \pi_\theta}(\hat{a} | s))\|^2\bigr], \quad \hat{a} \sim \pi_\theta(\cdot|s),
\label{eq:lk}
\end{equation}
with the drifting field $\mathbf{V}_{P_K, \pi_\theta}$ taking the empirical set $P_K$ as the positive set and $\pi_\theta$ as the negative distribution.

\begin{proposition}[Bounded approximation error of $\mathcal{L}_{\text{top-}K}$ to $\mathcal{L}_{\mathrm{PI}}$]
\label{prop:topk_limit}
Let $\rho := K/N$. With bounded Lipschitz kernel $k$ of Sec.~\ref{sec:prelim_drifting}, assume $Q_\phi(s, A)$ for $A \sim \pi_{\mathrm{old}}(\cdot | s)$ admits a strictly positive density at its $(1-\rho)$-quantile $q^\rho(s)$. Define the $\rho$-level-set tilting $\tilde{\pi}^\rho(a | s) := \rho^{-1}\, \mathds{1}[Q_\phi(s, a) \geq q^\rho(s)]\, \pi_{\mathrm{old}}(a | s)$. As $N \to \infty$ with $\rho$ fixed,
\begin{equation}
\mathcal{L}_{\text{top-}K}(\theta) \;\to\; \mathcal{L}_{\mathrm{drift}}(\theta;\, \tilde{\pi}^\rho, \pi_\theta), \qquad
\bigl|\mathcal{L}_{\mathrm{drift}}(\theta;\, \tilde{\pi}^\rho, \pi_\theta) - \mathcal{L}_{\mathrm{PI}}(\theta)\bigr| \;\leq\; C\, \overline{\mathrm{TV}}\bigl(\tilde{\pi}^\rho, \pi^+\bigr),
\label{eq:topk_bound}
\end{equation}
for any $\rho \in (0, 1]$ and a finite constant $C$ depending only on the kernel and action-space diameter, where $\overline{\mathrm{TV}}(p, q) := \mathbb{E}_s[\mathrm{TV}(p(\cdot | s), q(\cdot | s))]$. The derivation combines Glivenko-Cantelli on the empirical $(1-\rho)$-quantile with TV-Lipschitz continuity of the kernel mean shift; full proof in Appendix~\ref{app:topk_limit_proof}.
\end{proposition}

To anchor $\pi_\theta$ to the behavior distribution $\pi_\beta$, we additionally use a behavior cloning (BC) drift loss with empirical positives drawn from the replay buffer $\mathcal{D}$:
\begin{equation}
\mathcal{L}_{\mathrm{BC}}(\theta) = \mathcal{L}_{\mathrm{drift}}(\theta;\, \mathcal{D}, \pi_\theta) \;=\; \mathbb{E}_{(s, a) \sim \mathcal{D},\, \epsilon}\bigl[\|\hat{a} - \mathrm{sg}(\hat{a} + \mathbf{V}_{\mathcal{D}, \pi_\theta}(\hat{a} | s))\|^2\bigr], \quad \hat{a} \sim \pi_\theta(\cdot|s).
\label{eq:lbc}
\end{equation}
Although $\mathcal{L}_{\mathrm{BC}}$ and $\mathcal{L}_{\text{top-}K}$ originate from different objectives — data anchoring versus $Q$-maximizing policy improvement — they are both instances of $\mathcal{L}_{\mathrm{drift}}(\theta;\, p, q)$ (Eq.~\eqref{eq:Ldrift_general}), sharing the same negative $q = \pi_\theta$ and differing only in the empirical positive set:
\begin{itemize}[leftmargin=*]
    \item $\mathcal{L}_{\mathrm{BC}}(\theta) = \mathcal{L}_{\mathrm{drift}}(\theta;\, \mathcal{D}, \pi_\theta)$\,: replay buffer $\mathcal{D}$ as the positive (data anchor);
    \item $\mathcal{L}_{\text{top-}K}(\theta) = \mathcal{L}_{\mathrm{drift}}(\theta;\, P_K, \pi_\theta)$\,: top-$K$ self-generated set $P_K(s)$ as the positive ($Q$-improvement).
\end{itemize}
The combined loss $\mathcal{L}(\theta) = \mathcal{L}_{\mathrm{BC}}(\theta) + \lambda\, \mathcal{L}_{\text{top-}K}(\theta)$ with $\lambda > 0$ anchors $\pi_\theta$ to $\pi_\beta$ while pushing it toward higher-$Q$ regions. The shared drift form makes implementation simple: a single drift-loss routine evaluates both terms by swapping the positive set between $\mathcal{D}$ and $P_K(s)$. Algorithm~\ref{alg:dfp} summarizes the full online fine-tuning procedure.

\begin{algorithm}[t]
\caption{Drifting Field Policy (DFP), online fine-tuning}
\label{alg:dfp}
\begin{algorithmic}
\State \textbf{Input:} BC-pretrained policy $\pi_\theta(\cdot | s) = [f_\theta(\cdot, s)]_\# p_\epsilon$ and critic $Q_\phi$, replay buffer $\mathcal{D}$ initialized with offline data, $N$ candidates for $\mathcal{L}_{\text{top-}K}$, $N'$ candidates for best-of-$N'$ execution~\cite{ghasemipour2021:EMAQ}, $N_{\text{gen}}$ samples of  $\hat{a} \sim \pi_\theta(\cdot | s)$
\State Initialize old policy $\pi_{\mathrm{old}} \leftarrow \pi_\theta$
\For{online training step $k = 1, 2, \ldots$}
    \State Observe $s_k$; draw $\{a^{(i)}\}_{i=1}^{N'} \sim \pi_\theta(\cdot | s_k)$ and execute $a^\star_k \leftarrow \arg\max_{j} Q_\phi(s_k, a^{(j)})$
    \State Receive $s_{k+1}, r_k$; append $(s_k, a^\star_k, r_k, s_{k+1})$ to $\mathcal{D}$
    \State Sample mini-batch $\{(s^{(b)}, a^{(b)})\}_{b=1}^B \sim \mathcal{D}$
    \For{$b = 1, 2, \ldots, B$} \Comment{in parallel}
        \State Generate $\hat{a}^{(b,1)}, \ldots, \hat{a}^{(b,N_{\text{gen}})}  \sim \pi_\theta(\cdot | s^{(b)})$ \Comment{$\hat{a}$ shared by both $\mathcal{L}_{\mathrm{BC}}$ and $\mathcal{L}_{\text{top-}K}$}
        \State Draw $N$ candidates $\tilde{a}^{(b,1)}, \ldots, \tilde{a}^{(b,N)} \sim \pi_{\mathrm{old}}(\cdot | s^{(b)})$ 
        \State Select $P_K(s^{(b)}) \;=\; \bigl\{\tilde{a}^{(b, j)} \;:\; j \in \mathrm{argTopK}_{j\in[N]} \, Q_\phi(s^{(b)}, \tilde{a}^{(b,j)})\bigl\}$
    \EndFor
    \State Update $\theta$ by minimizing $\mathcal{L}_{\mathrm{BC}}(\theta) + \lambda\, \mathcal{L}_{\text{top-}K}(\theta)$ via Eqs.~\eqref{eq:lbc},~\eqref{eq:lk}
    \State Update $\phi$ via the Bellman backup of Eq.~\eqref{eq:critic_td}
    \State Update old policy: $\theta_{\mathrm{old}} \leftarrow \tau_\mathrm{EMA}\, \theta + (1 - \tau_\mathrm{EMA})\, \theta_{\mathrm{old}}$
\EndFor
\end{algorithmic}
\end{algorithm}

\subsection{Why Drifting Models for RL Finetuning?}
\label{sec:why_drifting}
Drifting policies (DFP) and few-step diffusion or flow policies~\cite{Zhan2026:MVP, Park2025:FQL} share the pushforward representation $\pi_\theta(\cdot | s) = [f_\theta(\cdot, s)]_\# p_\epsilon$ but parameterize $f_\theta$ in structurally different ways. Drifting policies parameterize $f_\theta$ \emph{directly} as a single-pass network whose output is the action itself. Diffusion and flow policies parameterize $f_\theta$ \emph{indirectly} through a time-indexed velocity field $v_\theta(a(t), t, s)$ along stochastic processes or their ODE counterparts~\cite{Ho:2020DDPM, Lipman:2023FM}; few-step variants~\cite{geng2025:mean, frans2024:shortcut, Song:2023CM} amortize the velocity integration along the trajectory into one or a few network calls but retain the time-indexed velocity field parameterization. We highlight two consequences of this choice for RL finetuning.


\paragraph{Probability space descent vs.\ velocity-field re-fitting.}
In the ideal nonparametric view, DFP corresponds to a WGF descent direction on \emph{probability space}: the drifting field transports policy samples toward $\pi^+$ along the steepest-descent direction of $\mathrm{KL}(\pi_\theta \,\|\, \pi^+)$ through the pushforward map $f_\theta(\epsilon, s)$. Diffusion and flow policies~\cite{wang2022:DiffusionQL, ding2024:QVPO}, including one-step variants~\cite{Zhan2026:MVP, sheng2026:mp1, Espinosa2025:shortcutpolicy}, do not admit such a direct descent: they retain a time-indexed velocity prediction $v_\theta(a(t), t, s)$ from which an action is generated through ODE integration~\cite{Song:2021ScoreSDE, Lipman:2023FM, Ho:2020DDPM, Song:2021DDIM}, and one-step inference (1 NFE)~\cite{Song:2023CM, frans2024:shortcut, geng2025:mean} does not eliminate this ODE dependence at training time. Shifting $\pi_\theta$ toward $\pi^+$ therefore requires globally re-fitting $v_\theta$, manifesting as two coupled burdens: a self-consistency constraint along the ODE (e.g., the MeanFlow identity~\cite{geng2025:mean}) that couples the velocity field across time in one-step variants, and trajectory-level credit assignment that spreads the $Q$-improvement signal across the integration path~\cite{Ren2025:DPPO,Mcallister2026:FPO,Li2026:QAM}. DFP sidesteps both: a single-pass pushforward map has no ODE structure, so output-level supervision realizes a direct WGF descent direction on probability space.

\paragraph{Built-in repulsion from current samples.}
The policy improvement loss of Eq.~\eqref{eq:ideal_pi_loss} pairs attraction toward $Q$-improvement targets ($\mathbf{V}^+_{\pi^+}$) with repulsion from the current policy's own samples ($\mathbf{V}^-_{\pi_\theta}$), which typically lie in lower-$Q$ regions than the targets. We conjecture that this built-in repulsion further pushes the policy away from low-$Q$ regions, a structural mechanism not present in diffusion or flow parameterizations.

\section{Related Work}
\label{sec:related_work}

\paragraph{Offline-to-Online RL.} A key challenge in offline-to-online RL is balancing the need to stay on the offline data support, avoiding out-of-distribution actions, against the exploration needed for online improvement. Existing methods address this intricate balance through (i)~conservative critic regularization~\cite{kumar2020:CQL, kostrikov2022:IQL, nakamoto2023:CalQL}, (ii)~behavior-cloning regularization on the policy~\cite{Nair2020:AWAC, tarasov2023:ReBRAC}, and (iii)~replay strategies mixing offline and online transitions~\cite{ball2023:RLPD, song2022:HybridRL}. While crucial for stability, these methods ultimately depend on the underlying policy parameterization to support both behavior cloning during offline pretraining and effective improvement under continually shifting targets during online fine-tuning, motivating our novel generative policy introduced below.

\paragraph{Few-Step Generative Policies.} In continuous-control RL, recent work replaces Gaussian policies~\cite{Lillicrap2016:DDPG, Fujimoto2018:TD3, Haarnoja2018:SAC} with generative policies that capture multimodal action distributions~\cite{wang2022:DiffusionQL, hansen2023:IDQL, ding2024:QVPO}. The field has since shifted toward few-step or one-step backbones for inference efficiency~\cite{Zhan2026:MVP, wang2026:reformulationMF, sheng2026:mp1, Ding2024:consistencypolicy, Espinosa2025:shortcutpolicy}. These backbones, however, still generate actions through a time-indexed SDE/ODE trajectory. When applied to RL fine-tuning, they require self-consistent updates across all intermediate timesteps to adapt to a shifting target distribution. We instead adopt a non-ODE policy based on \emph{drifting models}~\cite{Deng2026:Drifting} that generates actions in a single forward pass without any time variable, reducing adaptation to shifting only the network's output rather than re-aligning the entire generation trajectory.

\paragraph{Drifting Models.} Drifting models~\cite{Deng2026:Drifting} are one-step generative models that evolve the pushforward distribution at training time via stop-gradient regression onto a kernel-based mean-shift target, outperforming ODE-based few-step models~\cite{geng2025:mean, Song:2023CM, frans2024:shortcut} on image generation. Follow-up analyses relate the drift field to score matching~\cite{Lai2026:Unified, Turan2026:Drifting} and to the particle velocity of a Wasserstein-2 gradient flow under KDE-smoothed densities~\cite{Cao2026:GFD, He2026:Sinkhorn}. Concurrent works extend drifting models to control: Ada3Drift~\cite{xu2026:Ada3drift} and KDP~\cite{Puthumanaillam2026:KeyedDrift} only target offline imitation learning, while DBPO~\cite{gao2026:DBPO} targets offline-to-online fine-tuning but projects the drifting policy onto a unimodal Gaussian for PPO~\cite{Schulman2017:PPO} updates, sacrificing its expressiveness. We are the first to interpret drifting model RL fine-tuning as a Wasserstein-2 gradient flow descent on policy space and derive an algorithm that updates the drifting policy toward high-reward actions within a trust region, with a tractable surrogate of the otherwise intractable policy improvement loss.

\section{Experiments}
\label{sec:results}
Additional details and results are provided in the appendix: offline RL experiments (Appendix~\ref{app:offline_rl}), full offline-to-online results (Appendix~\ref{app:full-offline-to-online}), implementation details and hyperparameters (Appendix~\ref{app:hyperparameters}), and training and inference cost analysis (Appendix~\ref{app:timecost_analysis}).


\subsection{Experimental Setup}
\paragraph{Benchmarks.} We follow the experimental setup of Mean Velocity Policy (MVP)~\cite{Zhan2026:MVP}, evaluating on 3 tasks from Robomimic~\cite{robomimic2021} (\texttt{Lift}, \texttt{Can}, \texttt{Square}) under the multi-human (MH) datasets and 6 tasks from OGBench~\cite{park2024:OGbench}, with three tasks each from \texttt{cube-double} and \texttt{cube-triple} (\texttt{task-2/3/4} per environment). We further include three tasks from the \texttt{cube-quadruple} environment (\texttt{cube-quadruple-task-2/3/4}), resulting in a total of 12 tasks.


\paragraph{Baselines against DFP.}
We compare against the same baselines as MVP~\cite{Zhan2026:MVP}, spanning three categories: (i)~\emph{multi-step inference}: \textbf{BFN}~\cite{ghasemipour2021:EMAQ} and \textbf{QC-BFN}~\cite{Li2025:QC} train a multi-step BC flow policy and perform Best-of-$N$ extraction at inference (QC-BFN adds action-chunking on top of BFN); (ii)~\emph{distilled one-step}: \textbf{FQL}~\cite{Park2025:FQL} and \textbf{QC-FQL}~\cite{Li2025:QC} distill a one-step policy from a multi-step BC flow policy under a $Q$-maximization objective via the critic Jacobian $\nabla_a Q_\phi$ (QC-FQL adds action-chunking on top of FQL); (iii)~\emph{teacher-free one-step}: \textbf{MVP}~\cite{Zhan2026:MVP} trains a one-step MeanFlow~\cite{geng2025:mean} policy from scratch, jointly performing BC and policy improvement (the latter by imitating the best-of-$N$ action) on the same network. MVP is the closest baseline to DFP: both are teacher-free one-step policies with action-chunking, differing only in the generative policy parameterization (MeanFlow~\cite{geng2025:mean} vs. drifting models~\cite{Deng2026:Drifting}) and the corresponding training objectives, which we isolate in Sec.~\ref{subsec:main_ablation}.

\vspace{-1.0\baselineskip}
\paragraph{Hyperparameters.} We use $N=16$ candidate actions, with $K=2$ on Robomimic~\cite{robomimic2021} and $K=4$ on OGBench~\cite{park2024:OGbench}, and the top-$K$ loss coefficient $\lambda=0.5$ by default unless otherwise stated. Full hyperparameter configurations are listed in Appendix~\ref{app:hyperparameters}.

\subsection{Comparisons to Baselines}
\label{subsec:offline_to_online_results}


\newcommand{\val}[2]{#1{\scriptsize$\,\pm\,$#2}}
\begin{table}[t]
\centering
\footnotesize
\definecolor{stripegray}{gray}{0.96}
\caption{Success rate (\%) on Robomimic~\cite{robomimic2021} and OGBench~\cite{park2024:OGbench} tasks under the offline-to-online RL. Each cell reports \texttt{mean} $\pm$ \texttt{std} over 5 seeds. Best result per column in \textbf{bold}; second-best \underline{underlined}.}
\label{tab:main_results}
\setlength{\tabcolsep}{3pt}
\renewcommand{\arraystretch}{1.15}
\resizebox{\textwidth}{!}{%
\begin{tabular}{@{}lccccccccccccc@{}}
\toprule
& \multicolumn{3}{c}{\textbf{Robomimic}}
& \multicolumn{3}{c}{\textbf{Cube-double}}
& \multicolumn{3}{c}{\textbf{Cube-triple}}
& \multicolumn{3}{c}{\textbf{Cube-quadruple-100m}}
& \\
\cmidrule(lr){2-4}\cmidrule(lr){5-7}\cmidrule(lr){8-10}\cmidrule(lr){11-13}
\textbf{Method}
& \texttt{lift} & \texttt{square} & \texttt{can}
& \texttt{task2} & \texttt{task3} & \texttt{task4}
& \texttt{task2} & \texttt{task3} & \texttt{task4}
& \texttt{task2} & \texttt{task3} & \texttt{task4}
& \textbf{Avg.} \\
\midrule

BFN~\cite{ghasemipour2021:EMAQ}
& \val{97.6}{2} & \val{32.8}{8} & \val{82.0}{2} 
& \val{86.0}{5} & \val{88.8}{5} & \val{27.2}{8} 
& \val{7.6}{9} & \val{6.8}{3} & \val{0.0}{0} 
& \val{32.4}{21} & \val{0.0}{0} & \val{0.0}{0} 
& 38.4 \\

QC-BFN~\cite{Li2025:QC} 
& \val{99.6}{1} & \underline{\val{88.4}{4}} & \underline{\val{90.6}{3}} 
& \underline{\val{99.8}{0}} & \textbf{\val{99.8}{0}} & \val{92.6}{6} 
& \val{87.4}{10} & \underline{\val{80.8}{4}} & \val{33.4}{9} 
& \val{95.8}{2} & \val{63.2}{10} & \val{74.2}{11} 
& 83.8 \\
\midrule 
FQL~\cite{Park2025:FQL}
& \val{96.8}{2} & \val{10.8}{7} & \val{58.4}{8}
& \val{93.2}{8} & \val{91.2}{5} & \val{6.0}{6} 
& \val{0.4}{1} & \val{6.4}{8} & \val{0.0}{0}
& \val{0.0}{0} & \val{0.0}{0} & \val{0.0}{0} 
& 30.3 \\

QC-FQL~\cite{Li2025:QC} 
& \textbf{\val{100.0}{0}} & \val{72.0}{9} & \textbf{\val{94.4}{2}} 
& \textbf{\val{100.0}{0}} & \textbf{\val{99.8}{0}} & \textbf{\val{99.8}{0}} 
& \underline{\val{88.2}{2}} & \val{60.4}{12} & \underline{\val{51.4}{24} }
& \underline{\val{98.0}{2}} & \underline{\val{85.0}{7}} & \underline{\val{92.2}{7}} 
& \underline{86.8} \\
\midrule 
MVP~\cite{Zhan2026:MVP} 
& \underline{\val{99.8}{0}} & \val{79.4}{4} & \val{83.6}{5} 
& \val{98.4}{1} & \val{98.6}{1} & \val{94.8}{4} 
& \val{86.2}{4} & \val{57.2}{10} & \val{31.0}{20} 
& \val{96.6}{2} & \val{47.2}{30} & \val{91.2}{2} 
& 80.3 \\



\rowcolor{our_color!10}
\textcolor{our_color}{\textbf{DFP (Ours)}}
& \OC{\textbf{\val{100.0}{0}}} & \OC{\textbf{\val{93.2}{2}}} & \OC{\underline{\val{90.6}{3}}}
& \OC{\textbf{\val{100.0}{0}}} & \OC{\underline{\val{99.6}{1}}} & \OC{\underline{\val{99.6}{1}}}
& \OC{\textbf{\val{98.4}{1}}} & \OC{\textbf{\val{91.6}{2}}} & \OC{\textbf{\val{81.2}{6}}}
& \OC{\textbf{\val{99.6}{1}}} & \OC{\textbf{\val{96.6}{2}}} & \OC{\textbf{\val{99.0}{2}}}
& \OC{\textbf{95.8}} \\
\bottomrule
\end{tabular}%
}
\end{table}

As summarized in Tab.~\ref{tab:main_results}, DFP achieves the highest average success rate of $95.8\%$, ranking first on $9$ of $12$ tasks and second-best on the remaining $3$, outperforming the strongest baseline QC-FQL~\cite{Li2025:QC} ($86.8\%$) by $+9.0$ pp. Despite generating actions in a single forward pass, DFP surpasses even the best multi-step policy QC-BFN~\cite{Li2025:QC} ($83.8\%$, $+12.0$ pp), with the gap widening on the harder \texttt{cube-triple} and \texttt{cube-quadruple} splits where multi-step BC alone is insufficient. 

The most informative comparison is against MVP~\cite{Zhan2026:MVP}, the closest baseline: both methods train a single one-step policy with no multi-step teacher. DFP improves the average score by $+15.5$ pp over MVP~\cite{Zhan2026:MVP} and outperforms on every task, with particularly large gains on the multimodal long-horizon tasks (e.g., \texttt{cube-triple-task4}: $31.0 \!\to\! 81.2$; \texttt{cube-quadruple-task3}: $47.2 \!\to\! 96.6$). As shown by the training curves in Fig.~\ref{fig:benchmark}, DFP outperforms MVP~\cite{Zhan2026:MVP} consistently across time, demonstrating both faster convergence in the online stage and better behavior cloning capability in the offline stage. The full result table, including both offline and online phases, is in Appendix~\ref{app:full-offline-to-online}.

\begin{figure}[!ht]
\centering
\captionsetup[subfigure]{justification=centering, aboveskip=4pt, belowskip=4pt}
\subfloat[Robomimic-lift\label{subFig:main_lift}]
{\includegraphics[width = 0.32\textwidth,trim={0cm 0.2cm 0.0cm 0.2cm}, clip]{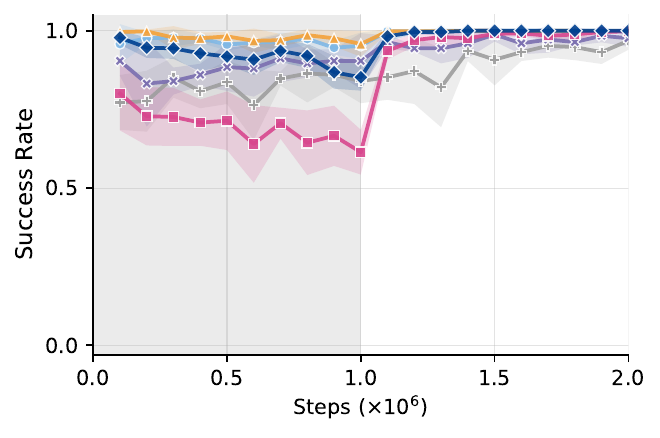}} 
\subfloat[Robomimic-square\label{subFig:main_square}]
{\includegraphics[width = 0.32\textwidth,trim={0cm 0.2cm 0.0cm 0.2cm}, clip]{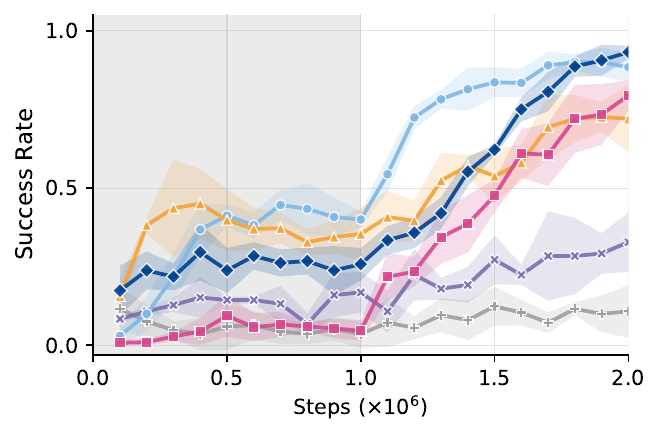}}
\subfloat[Robomimic-can\label{subFig:main_can}]
{\includegraphics[width = 0.32\textwidth,trim={0cm 0.2cm 0.0cm 0.2cm}, clip]{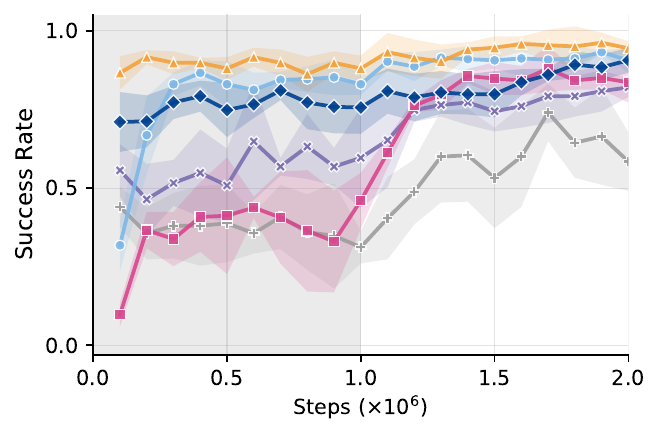}} \\[-0.5em]
\subfloat[Cube-double-task2\label{subFig:main_double-2}]
{\includegraphics[width = 0.32\textwidth,trim={0cm 0.2cm 0.0cm 0.2cm}, clip]{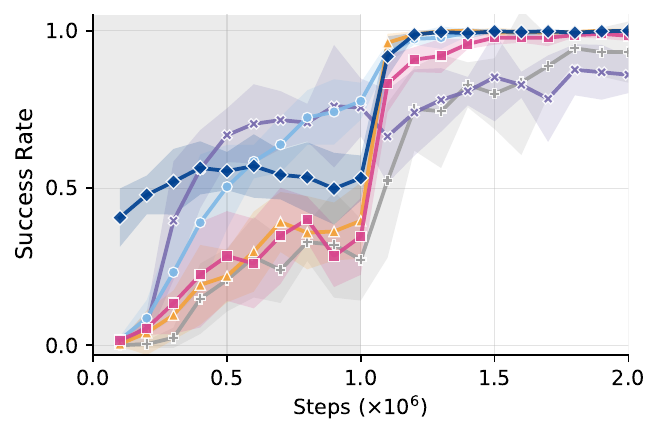}} 
\subfloat[Cube-double-task3\label{subFig:main_double-3}]
{\includegraphics[width = 0.32\textwidth,trim={0cm 0.2cm 0.0cm 0.2cm}, clip]{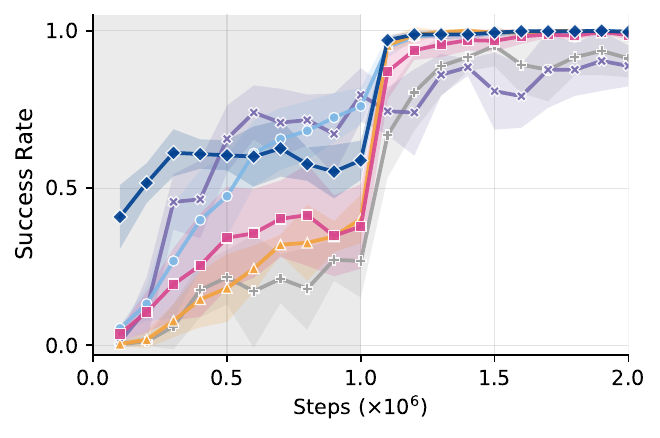}} 
\subfloat[Cube-double-task4\label{subFig:main_double-4}]
{\includegraphics[width = 0.32\textwidth,trim={0cm 0.2cm 0.0cm 0.2cm}, clip]{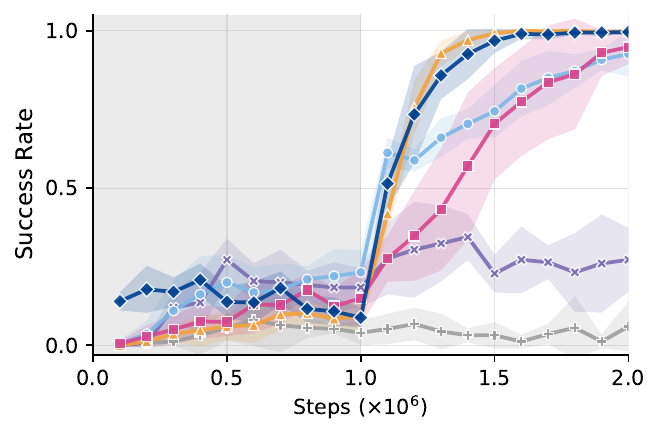}} \\[-0.5em]
\subfloat[Cube-triple-task2\label{subFig:main_triple-2}]
{\includegraphics[width = 0.32\textwidth,trim={0cm 0.2cm 0.0cm 0.2cm}, clip]{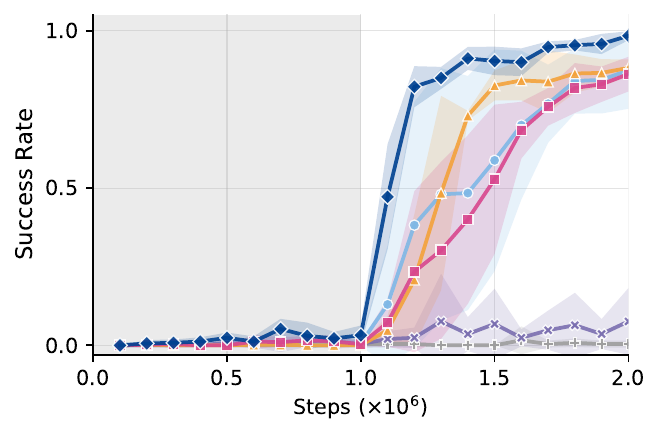}} 
\subfloat[Cube-triple-task3\label{subFig:main_triple-3}]
{\includegraphics[width = 0.32\textwidth,trim={0cm 0.2cm 0.0cm 0.2cm}, clip]{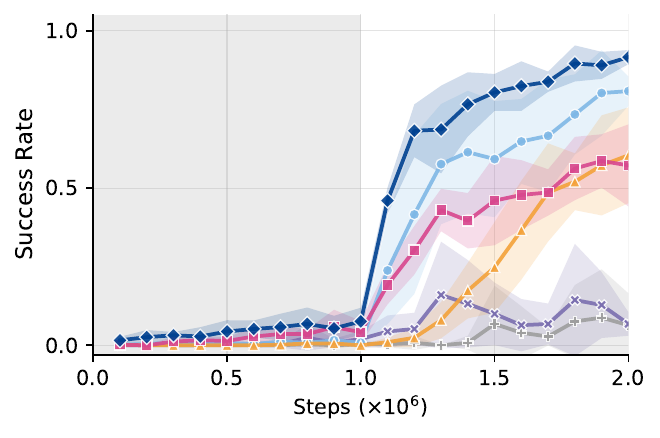}}
\subfloat[Cube-triple-task4\label{subFig:main_triple-4}]
{\includegraphics[width = 0.32\textwidth,trim={0cm 0.2cm 0.0cm 0.2cm}, clip]{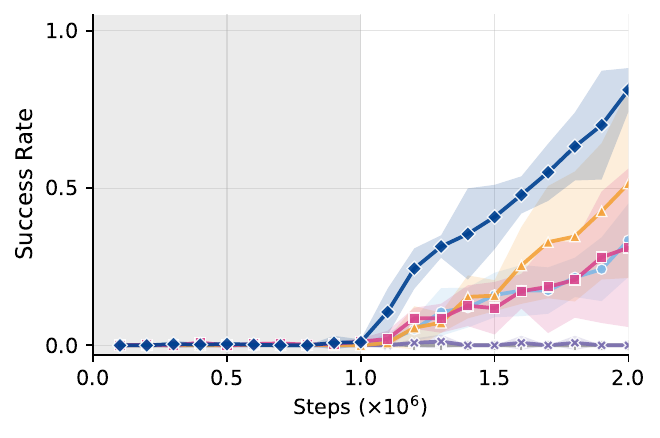}} \\[-0.5em]
\subfloat[Cube-quad-task2\label{subFig:main_quad-2}]
{\includegraphics[width = 0.32\textwidth,trim={0cm 0.2cm 0.0cm 0.2cm}, clip]{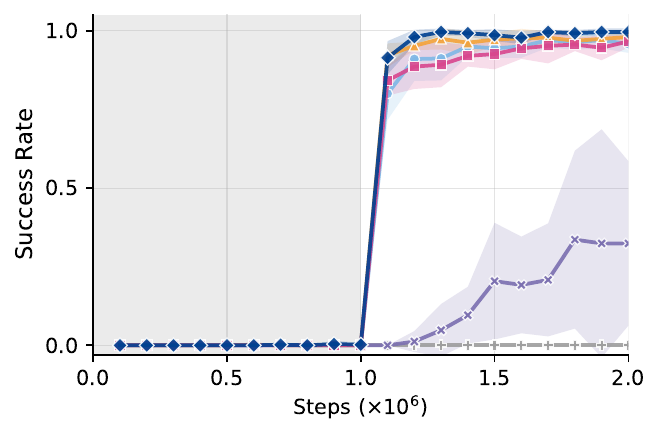}} 
\subfloat[Cube-quad-task3\label{subFig:main_quad-3}]
{\includegraphics[width = 0.32\textwidth,trim={0cm 0.2cm 0.0cm 0.2cm}, clip]{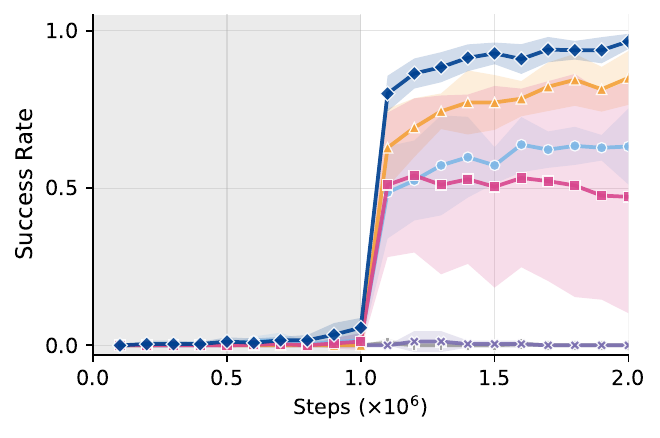}}
\subfloat[Cube-quad-task4\label{subFig:main_quad-4}]
{\includegraphics[width = 0.32\textwidth,trim={0cm 0.2cm 0.0cm 0.2cm}, clip]{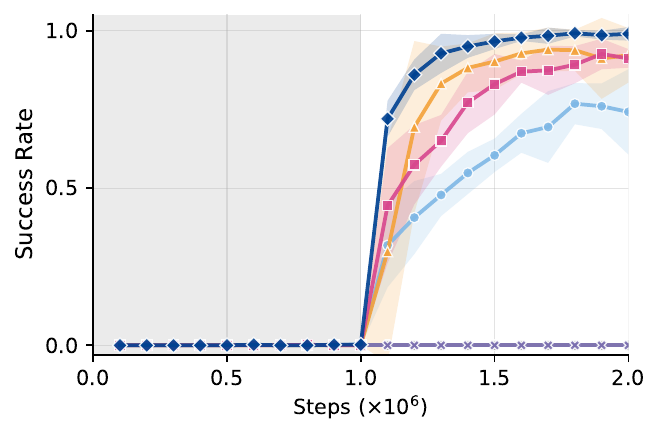}} \\[0.2em]
{\includegraphics[width = 0.75\textwidth, clip]{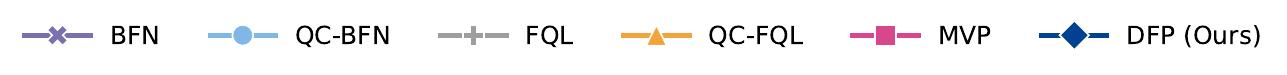}} 
\caption{\textbf{Success rate over training steps across Robomimic~\cite{robomimic2021} and OGBench~\cite{park2024:OGbench}.} Solid lines and shaded regions show the mean and 95\% confidence interval over five runs. Gray and white backgrounds indicate the offline and online phases, respectively.}
\label{fig:benchmark}
\end{figure}

\subsection{Analysis: Drifting vs. MeanFlow under Identical Loss}
\label{subsec:drifting_vs_meanflow}

To isolate the contribution of each design choice that distinguishes DFP from MVP~\cite{Zhan2026:MVP}, we ablate two factors along two axes: the policy backbone (MeanFlow~\cite{geng2025:mean} vs.~drifting model~\cite{Deng2026:Drifting}) and the training objective (BC only vs.~BC $+\, \mathcal{L}_{\text{top-}K}$). For the MeanFlow backbone, applying $\mathcal{L}_{\text{top-}K}$ reduces to applying its native MeanFlow identity loss to the top-$K$ candidates. The results are summarized in Tab.~\ref{tab:topk_loss_ablation}. Even without $\mathcal{L}_{\text{top-}K}$ (BC only), DFP yields a $+8.1$\,pp in average success rate over MVP, confirming the drifting parameterization's stronger behavior cloning capability. Adding $\mathcal{L}_{\text{top-}K}$ benefits the two backbones asymmetrically: MeanFlow gains only marginally ($80.3\% \!\to\! 82.7\%$, $+2.4$\,pp), while the drifting backbone gains substantially ($88.4\% \!\to\! 95.8\%$, $+7.4$\,pp). 

This asymmetry follows from the parameterization-level distinction discussed in Sec.~\ref{sec:wgf_theory}. On the drifting backbone, the top-$K$ supervision is applied directly to generated actions through the pushforward map $f_\theta$ in a single forward pass, approximating a WGF descent direction toward the high-$Q$ candidates. On MeanFlow, the same $K$ targets must instead enter through its identity loss~\cite{geng2025:mean} and re-fit $v_\theta$ globally, where the self-consistency constraint couples $v_\theta$ across time and trajectory-level credit assignment spreads the supervision along the integration path; the same targets therefore yield a smaller shift in $\pi_\theta$. See Appendix~\ref{app:lk-asymmetry} for training curve comparisons.

\subsection{Ablation Study}
\label{subsec:main_ablation}
Refer to Sec.~\ref{app:k-ablation-curves} and Sec.~\ref{app:lambda_ablation} in the appendix for more comprehensive ablation results.

\begin{table}[H]
\centering
\footnotesize
\definecolor{stripegray}{gray}{0.96}
\caption{Ablation on the $\mathcal{L}_{\text{top-}K}$ for MVP~\cite{Zhan2026:MVP} and DFP. Success rate (\%) on Robomimic and OGBench tasks under the offline-to-online protocol. Each cell reports \texttt{mean} $\pm$ \texttt{std} over 5 seeds. Best result per column in \textbf{bold}; second-best \underline{underlined}.}
\label{tab:topk_loss_ablation}
\setlength{\tabcolsep}{3pt}
\renewcommand{\arraystretch}{1.15}
\resizebox{\textwidth}{!}{%
\begin{tabular}{@{}llcccccccccccccc@{}}
\toprule
& & 
& \multicolumn{3}{c}{\textbf{Robomimic}}
& \multicolumn{3}{c}{\textbf{Cube-double}}
& \multicolumn{3}{c}{\textbf{Cube-triple}}
& \multicolumn{3}{c}{\textbf{Cube-quadruple-100m}}
& \\
\cmidrule(lr){4-6}\cmidrule(lr){7-9}\cmidrule(lr){10-12}\cmidrule(lr){13-15}
\textbf{Method} & \textbf{Backbone} & \textbf{$\mathcal{L}_{\text{top-}K}$}
& \texttt{lift} & \texttt{square} & \texttt{can}
& \texttt{task2} & \texttt{task3} & \texttt{task4}
& \texttt{task2} & \texttt{task3} & \texttt{task4}
& \texttt{task2} & \texttt{task3} & \texttt{task4}
& \textbf{Avg.} \\
\midrule




MVP~\cite{Zhan2026:MVP} 
& \textbf{MeanFlow} & \xmark 
& \underline{\val{99.8}{0}} & \val{79.4}{4} & \val{83.6}{5} 
& \val{98.4}{1} & \val{98.6}{1} & \val{94.8}{4} 
& \val{86.2}{4} & \val{57.2}{10} & \val{31.0}{20} 
& \val{96.6}{2} & \val{47.2}{30} & \val{91.2}{2} 
& 80.3 \\

MVP w/ $\mathcal{L}_{\text{top-}K}$
& \textbf{MeanFlow} & \cmark 
& \textbf{\val{100.0}{0}} & \val{81.6}{5} & \val{86.2}{6}
& \underline{\val{99.6}{1}} & \underline{\val{98.8}{1}} & \underline{\val{96.6}{1}}
& \val{78.0}{15} & \val{60.6}{13} & \val{30.4}{11}
& \val{97.2}{1} & \val{69.2}{15} & \val{94.2}{4} 
& 82.7 \\
\midrule

DFP w/o $\mathcal{L}_{\text{top-}K}$
& \textbf{Drifting} & \xmark 
& \textbf{\val{100.0}{0}} & \underline{\val{88.6}{2}} & \underline{\val{90.4}{5}}
& \val{99.2}{1} & \textbf{\val{99.6}{1}} & \val{96.0}{3} 
& \underline{\val{91.4}{3}} & \underline{\val{83.2}{4}} & \underline{\val{31.2}{7}}
& \underline{\val{97.6}{1}} & \underline{\val{88.8}{3}} & \underline{\val{95.2}{3}}
& \underline{88.4} \\

\rowcolor{our_color!10}
\textcolor{our_color}{\textbf{DFP (Ours)}}
& \OC{\textbf{Drifting}} & \OC{\cmark}
& \OC{\textbf{\val{100.0}{0}}} & \OC{\textbf{\val{93.2}{2}}} & \OC{\textbf{\val{90.6}{3}}}
& \OC{\textbf{\val{100.0}{0}}} & \OC{\textbf{\val{99.6}{1}}} & \OC{\textbf{\val{99.6}{1}}}
& \OC{\textbf{\val{98.4}{1}}} & \OC{\textbf{\val{91.6}{2}}} & \OC{\textbf{\val{81.2}{6}}}
& \OC{\textbf{\val{99.6}{1}}} & \OC{\textbf{\val{96.6}{2}}} & \OC{\textbf{\val{99.0}{2}}}
& \OC{\textbf{95.8}} \\
\bottomrule
\end{tabular}%
}
\vspace{-1.0\baselineskip}
\end{table}
\newcolumntype{C}[1]{>{\centering\arraybackslash}p{#1}}
\begin{table}[H]
\centering
\scriptsize
\setlength{\tabcolsep}{4pt}
\renewcommand{\arraystretch}{1.15}
\begin{minipage}[t]{0.46\textwidth}
\centering
\caption{Top-$K$ loss weight $\lambda$ ablation on cube-quadruple tasks, \texttt{mean} and \texttt{std} over 5 seeds. \textbf{Bold}/\underline{underline}: best/second-best.}
\label{tab:beta_ablation}
\begin{tabular}{@{}lcccc@{}}
\toprule
Cube-4. & \textbf{task 2} & \textbf{task 3} & \textbf{task 4} \\
\midrule
$\lambda=0.1$ & \val{99.0}{1}           & \val{86.2}{8}          & \val{96.4}{2} \\
$\lambda=0.5$ & \textbf{\val{99.6}{1}}  & \textbf{\val{96.6}{2}} & \underline{\val{99.0}{2}} \\
$\lambda=1.0$ & \val{98.4}{1}           & \textbf{\val{96.6}{2}} & \textbf{\val{99.2}{0}} \\
$\lambda=5.0$ & \val{99.2}{0}           & \val{94.2}{4}          & \val{98.0}{2} \\


\bottomrule
\end{tabular}
\end{minipage}
\hspace{0.04\textwidth}
\begin{minipage}[t]{0.46\textwidth}
\centering
\caption{$K$ ablation. Success rate (\%) averaged over tasks per environment and 5 seeds. \textbf{Bold}/\underline{underline}: best/second-best.}
\label{tab:topk_ablation}
\begin{tabular}{@{}lccccc@{}}
\toprule
& \textbf{Robo.} & \textbf{Cube-2.} & \textbf{Cube-3.} & \textbf{Cube-4.} & \textbf{Avg.} \\
\midrule
$K=1$ & \underline{95.0} & \underline{99.7} & 54.5 & 94.2 & 85.8 \\
$K=2$ & \textbf{94.6} & \textbf{100.0} & \underline{76.2} & \underline{96.8} & \underline{91.9} \\
$K=4$ & 93.9 & \underline{99.7} & \textbf{90.4} & \textbf{98.4} & \textbf{95.6} \\
$K=8$ & 88.6 & 99.8 & 85.5 & 96.2 & 92.5 \\
\bottomrule
\end{tabular}
\end{minipage}%
\hfill
\end{table}
\paragraph{Effect of the top-$K$ loss weight $\lambda$.}
We perform an ablation study on the $\mathcal{L}_{\text{top-}K}$ weight $\lambda$. As reported in Tab.~\ref{tab:beta_ablation}, we vary $\lambda \in \{0.1, 0.5, 1.0, 5.0\}$ and observe that DFP is robust to the weight.
\paragraph{Effect of the number of positives $K$.}
With the candidate pool size fixed at $N = 16$, we vary $K \in \{1, 2, 4, 8\}$ in $\mathcal{L}_{\text{top-}K}$, reported in Tab.~\ref{tab:topk_ablation}. While a sweet spot exists, $\mathcal{L}_{\text{top-}K}$ is overall robust to $K$: even the worst setting ($K=1$, $85.8\%$) is already comparable to the best baseline (QC-FQL, $86.8\%$). 
\section{Conclusion}
\label{sec:conclusion}
We present \emph{Drifting Field Policy} ({DFP}), a novel one-step generative policy grounded in a Wasserstein-2 gradient flow interpretation of drifting model training. The resulting top-$K$ drift loss updates the policy toward high-reward actions within a trust region around the previous policy, with a tractable surrogate of the policy improvement objective. On Robomimic~\cite{robomimic2021} and OGBench~\cite{park2024:OGbench} benchmarks, DFP achieves $95.8\%$ average success rate, outperforming all baselines including multi-step variants. Ablations attribute the gain to both the drifting parameterization and the top-$K$ supervision, whose synergy is unique to the drifting backbone.

\paragraph{Limitations.}

DFP is currently studied on simulated continuous-action manipulation tasks; extensions to high-dimensional observations and sim-to-real deployment remain future work. Performance of the proposed top-$K$ loss depends on the quality of the learned critic, a limitation shared with actor-critic methods broadly. Finally, while we provide structural analyses and supporting ablations for the advantages of our non-ODE parameterization over ODE-based policies, a deeper theoretical understanding of these distinctions remains an open question.
\bibliographystyle{abbrv}
\bibliography{ref}
\clearpage
\newpage

\appendix
\section{Experimental Details}
\label{app:experimental-details}

\subsection{Environment Descriptions}
\label{app:environment-descriptions}

We evaluate on $12$ manipulation tasks drawn from the Robomimic benchmark~\citep{robomimic2021} and the OGBench manipulation suite~\citep{park2024:OGbench}. Robomimic uses a $7$-DoF Franka Emika Panda arm, while the OGBench cube environments use a $6$-DoF UR5e arm with a Robotiq $2$F-$85$ parallel-jaw gripper. All tasks employ sparse, completion-style rewards. Fig.~\ref{fig:env_visualization} visualizes all $12$ tasks.

\begin{figure}
\centering
\captionsetup[subfigure]{justification=centering, aboveskip=4pt, belowskip=4pt}

\subfloat[Lift\label{subFig:env_lift}]
{\includegraphics[width = 0.32\textwidth]{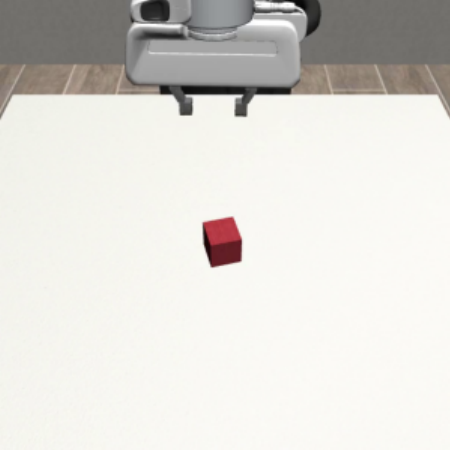}} 
\subfloat[Can\label{subFig:env_can}]
{\includegraphics[width = 0.32\textwidth]{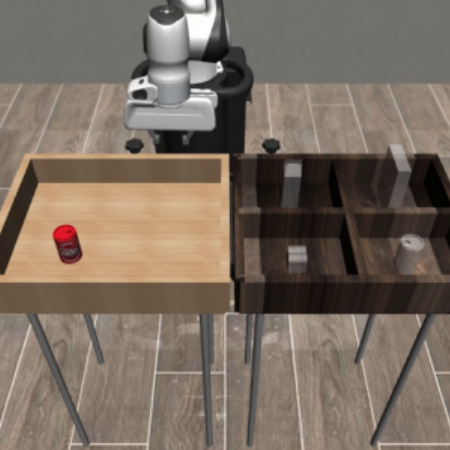}}
\subfloat[Square\label{subFig:env_square}]
{\includegraphics[width = 0.32\textwidth]{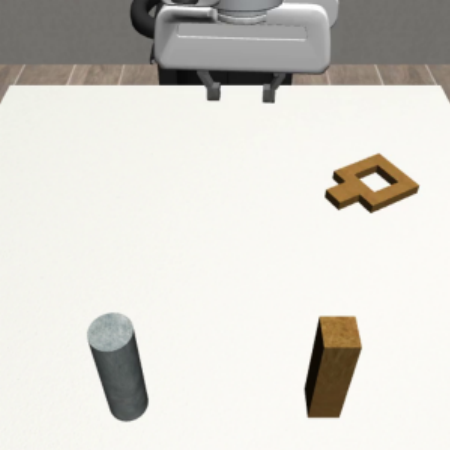}} \\[-0.5em]

\subfloat[Cube-double-task2\label{subFig:env_double-2}]
{\includegraphics[width = 0.32\textwidth]{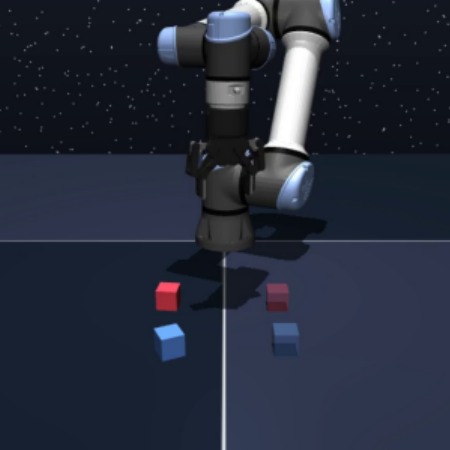}} 
\subfloat[Cube-double-task3\label{subFig:env_double-3}]
{\includegraphics[width = 0.32\textwidth]{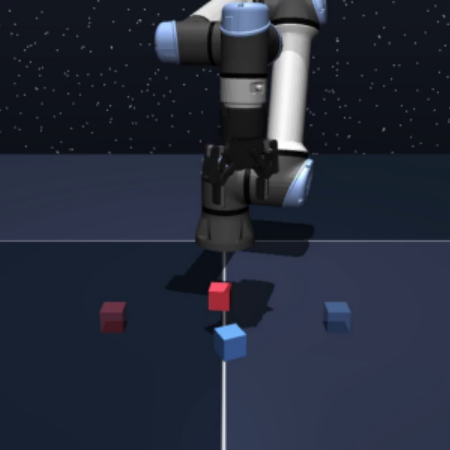}} 
\subfloat[Cube-double-task4\label{subFig:env_double-4}]
{\includegraphics[width = 0.32\textwidth]{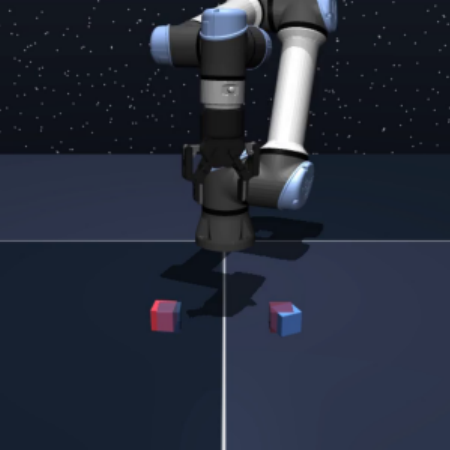}} \\[-0.5em]

\subfloat[Cube-triple-task2\label{subFig:env_triple-2}]
{\includegraphics[width = 0.32\textwidth]{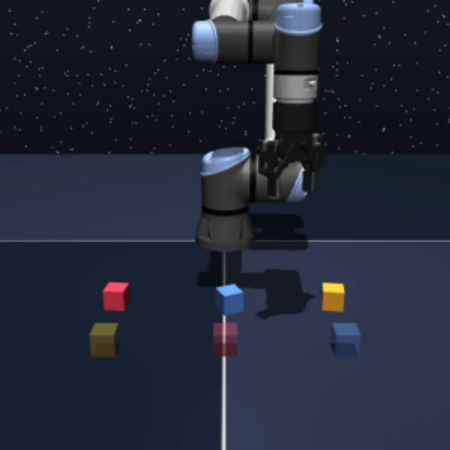}} 
\subfloat[Cube-triple-task3\label{subFig:env_triple-3}]
{\includegraphics[width = 0.32\textwidth]{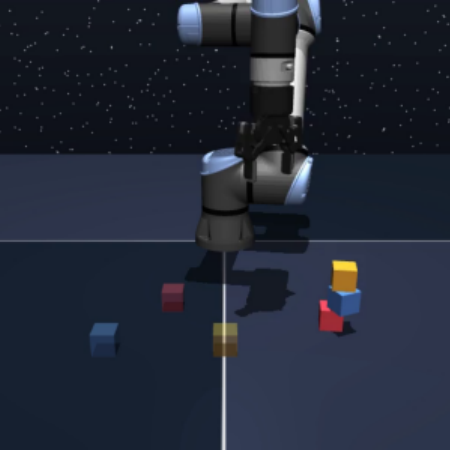}}
\subfloat[Cube-triple-task4\label{subFig:env_triple-4}]
{\includegraphics[width = 0.32\textwidth]{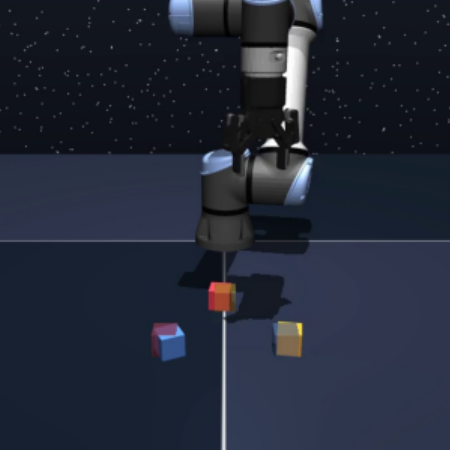}} \\[-0.5em]

\subfloat[Cube-quadruple-task2\label{subFig:env_quad-2}]
{\includegraphics[width = 0.32\textwidth]{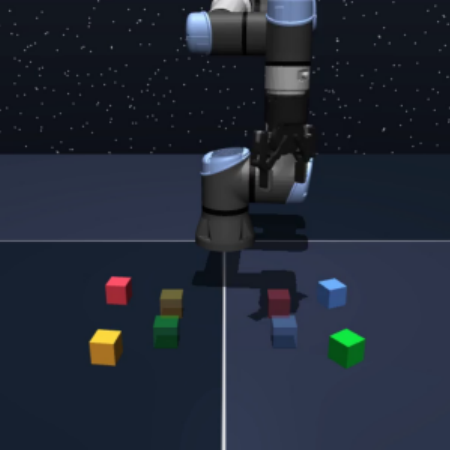}} 
\subfloat[Cube-quadruple-task3\label{subFig:env_quad-3}]
{\includegraphics[width = 0.32\textwidth]{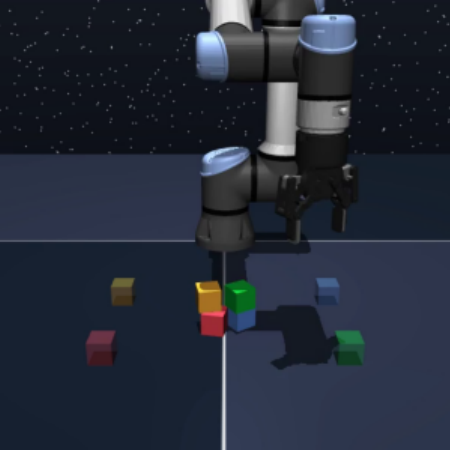}}
\subfloat[Cube-quadruple-task4\label{subFig:env_quad-4}]
{\includegraphics[width = 0.32\textwidth]{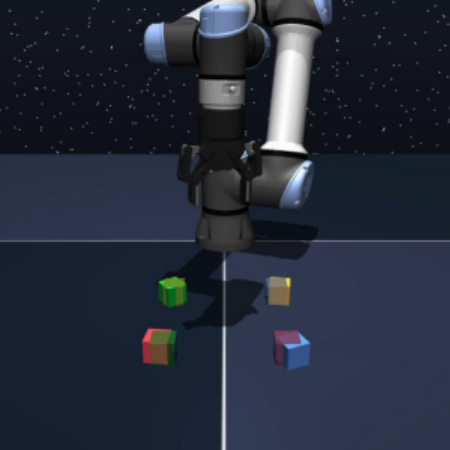}}

\caption{\textbf{Visualization of the 12 manipulation tasks used in our experiments.} The top row shows the Robomimic Multi-Human tasks (Lift, Can, Square), and the remaining rows show the OGBench Cube environments with $N \in \{2, 3, 4\}$ cubes. Each panel depicts the initial configuration of a representative episode.}
\label{fig:env_visualization}
\end{figure}

\subsubsection{Robomimic (Multi-Human)}
\label{app:robomimic_details}

\paragraph{Datasets.}
We use the Multi-Human (MH) variant of each Robomimic task, collected via teleoperation through the RoboTurk platform by six operators stratified by skill (two ``worse'', two ``okay'', and two ``better''). Each operator contributes $50$ successful trajectories, yielding $300$ heterogeneous demonstrations per task. The intentionally non-uniform action distribution makes MH a harder offline-learning setting than the cleaner Proficient-Human variants.

\paragraph{Tasks.}
Every task is solved by a single $7$-DoF Franka Emika Panda arm in a tabletop workspace and terminates upon success, with a sparse reward.
\begin{itemize}[leftmargin=1.2em]
    \item \textbf{Lift} (Fig.~\ref{fig:env_visualization}a) -- Grasp a randomly placed cube and raise it above a height threshold. Tests basic grasping.
    \item \textbf{Can} (Fig.~\ref{fig:env_visualization}b) -- Transport a coke can from one bin to a smaller target bin, requiring coordinated reach, grasp, and release.
    \item \textbf{Square} (Fig.~\ref{fig:env_visualization}c) -- Pick up a square nut and thread it onto a peg with sub-centimeter tolerance. The most precision-sensitive task in the suite.
\end{itemize}

\subsubsection{OGBench Cube Environments}
\label{app:ogbench_details}

\paragraph{Common setup.}
The \texttt{cube-double}, \texttt{cube-triple}, and \texttt{cube-quadruple} 
environments share an identical $6$-DoF UR5e arm with a Robotiq $2$F-$85$ 
gripper but vary the number of cubes $N \in \{2, 3, 4\}$. We use the 
\texttt{play-singletask-task[N]-v0} variants, which fix an evaluation goal 
and relabel the unstructured \texttt{play} dataset with the corresponding 
reward function for offline pre-training. Among the five canonical evaluation 
goals (\texttt{task1}--\texttt{task5}) provided by each environment, we focus 
on \texttt{task2}--\texttt{task4}, which together cover representative skills 
including multi-object pick-and-place, structural manipulation, and 
combinatorial rearrangement.
The reward is semi-sparse, defined as $r = -n_{\text{wrong}}$, where 
$n_{\text{wrong}}$ counts cubes whose position has not yet reached its 
target within the environment's success tolerance; an episode terminates 
only when all $N$ cubes simultaneously satisfy the goal criterion. 
Following OGBench's convention, success is determined by cube positions 
only, ignoring orientation.

\paragraph{cube-double} (Fig.~\ref{fig:env_visualization}d--f) covers three 
multi-object skills corresponding to the OGBench-defined goals 
\emph{double-pnp1} (\texttt{task2}), \emph{double-pnp2} (\texttt{task3}), 
and \emph{swap} (\texttt{task4}), where the last requires exchanging the 
positions of two occupied cubes.

\paragraph{cube-triple} (Fig.~\ref{fig:env_visualization}g--i) requires 
composing pick-and-place primitives in non-trivial sequences: 
\emph{triple-pnp} rearrangement of three cubes (\texttt{task2}), 
\emph{pnp-from-stack} which involves manipulating cubes from a stacked 
configuration (\texttt{task3}), and \emph{cyclic permutation} of three 
cubes (\texttt{task4}).

\paragraph{cube-quadruple} (Fig.~\ref{fig:env_visualization}j--l) amplifies 
long-horizon coordination, requiring up to four sequential pick-and-place 
subtasks per goal (Table 1 of \cite{park2024:OGbench}): 
\emph{quadruple-pnp} rearrangement of four cubes (\texttt{task2}), 
\emph{pnp-from-square} which manipulates cubes from a square configuration 
(\texttt{task3}), and a \emph{4-cycle} permutation (\texttt{task4}) that 
cannot be decomposed into fewer than three pairwise swaps.

The progression from \texttt{cube-double} to \texttt{cube-quadruple} is 
intentionally combinatorial: the same primitives (pick-and-place, swap, 
cycle) recur, but the number of required subtasks scales with $N$ 
(up to 1, 2, 4 atomic behaviors for $N=2, 3, 4$ respectively, per OGBench 
Table 1), with the longest evaluation task requiring approximately 400 
environment steps. This provides a controlled benchmark for long-horizon 
sequential reasoning and credit assignment.

\subsection{Baseline Method Details}
\label{app:baseline-method-details}

We compare DFP against five recent generative model based policies for offline-to-online RL. All baselines share the off-policy actor-critic skeleton of Eqs.~\eqref{eq:critic_td}-\eqref{eq:actor_loss} but differ in the policy parameterization and how the actor is supervised. We summarize each below; full hyperparameters are listed in Appendix~\ref{app:hyperparameters}.

\paragraph{FQL~\citep{Park2025:FQL}.}
Flow Q-Learning trains a multi-step flow-matching behavior policy $\mu_\theta(s, z)$ with the standard flow-matching objective and jointly distills it into a one-step student $\mu_\omega(s, z)$ that is optimized to maximize $Q_\phi$ under a distillation regularizer to $\mu_\theta$. The one-step student is used at inference, eliminating iterative integration.

\paragraph{BFN~\citep{ghasemipour2021:EMAQ}.}
Best-of-$N$ is a critic-guided action-selection wrapper around a behavior-cloned multi-step flow policy: at every action call, $N$ candidates are drawn from the BC policy and the one with the highest $Q_\phi(s, a^{(j)})$ is executed. We follow the formulation of EMaQ~\citep{ghasemipour2021:EMAQ}, which derives this scheme from the expected-max Bellman backup and uses the BC policy as the proposal. BFN exploits the critic only at execution time; the policy itself is never updated by $Q_\phi$.

\paragraph{QC-BFN~\citep{Li2025:QC}.}
Q-chunking~\cite{Li2025:QC} applied to BFN. The actor predicts a temporally extended sequence of $H$ future actions rather than a single action, and RL is run directly in the chunked action space with an unbiased $H$-step Bellman backup that mitigates exploration in long-horizon sparse-reward tasks. The BFN best-of-$N$ wrapper is then applied on chunks at execution.

\paragraph{QC-FQL~\citep{Li2025:QC}.}
Q-chunking~\cite{Li2025:QC} applied to FQL. The chunked $h$-step Bellman backup of QC-BFN is retained, but the actor is the one-step FQL student rather than the multi-step BC flow, paying only one-step inference cost while inheriting the chunked exploration benefit.

\paragraph{MVP~\citep{Zhan2026:MVP}.}
Mean Velocity Policy parameterizes the actor as a MeanFlow~\citep{geng2025:mean} velocity field $v_\theta(a^{(t)}, t, r, s)$ that models the \emph{mean} velocity over an interval $[r, t]$, enabling a single Euler step from noise to action at inference. MVP introduces an instantaneous velocity constraint (IVC) at the interval boundary as an auxiliary regression loss to disambiguate the otherwise under-determined ODE. Online supervision relies on best-of-$N$ environment rollouts: $N$ actions are sampled from the current actor at every state, the highest-$Q$ candidate is executed and stored, and the velocity field is regressed onto these critic-selected actions via standard MeanFlow self-consistency. MVP is the most direct empirical comparison for DFP since it shares the one-step regime and the best-of-$N$ online recipe but uses an ODE-based backbone and single-positive supervision.

\paragraph{MVP w/ $\mathcal{L}_{\text{top-}K}$.}
A controlled ablation we introduce to isolate the effect of top-$K$ supervision from the choice of backbone. The actor parameterization, IVC loss, and best-of-$N$ rollout are inherited from MVP; the only modification is to extend the online actor supervision from a single critic-selected target to the same top-$K$ multi-positive set $\mathcal{P}_K(s)$ used in DFP, regressing the MeanFlow velocity field onto each of the $K$ positives. This isolates the question: does top-$K$ supervision transfer to an ODE-based one-step backbone, or does the gain require the drifting parameterization? Sec.~\ref{sec:results} reports that the gain is largely backbone-specific, supporting the latter.

\subsection{Hyperparameters}
\label{app:hyperparameters}
Tab.~\ref{tab:hyperparameters} lists the hyperparameters used by DFP. The shared block (top) covers the actor-critic backbone, and the DFP-specific block (bottom) lists the additional hyperparameters introduced for the drifting field and the top-$K$ surrogate loss. For all baselines, we follow the configurations reported in their original papers~\citep{Park2025:FQL, ghasemipour2021:EMAQ, Li2025:QC, Zhan2026:MVP}.
\begin{table}[h]
\centering
\caption{\textbf{Detailed hyperparameters.}}
\label{tab:hyperparameters}
\begin{tabular}{lc}
\toprule
\textbf{Parameter} & \textbf{Value} \\
\midrule
\multicolumn{2}{l}{\textbf{\textit{Shared}}} \\
\addlinespace[2pt]
Batch size                                    & 256 \\
Discount factor ($\gamma$)                    & 0.99 \\
Optimizer                                     & Adam \\
Learning rate                                 & $3 \times 10^{-4}$ \\
Target network update rate ($\tau$)           & $5 \times 10^{-3}$ \\
Number of offline training steps               & $1 \times 10^{6}$ (1M) \\
Number of online training steps               & $1 \times 10^{6}$ (1M) \\
Total gradient steps                          & $2 \times 10^{6}$ (2M) \\
Policy network width                          & 512 \\
Policy network depth                          & 4 hidden layers \\
Policy activation function                    & GELU \\
Policy layer normalization                    & False \\
Value network width                           & 512 \\
Value network depth                           & 4 hidden layers \\
Value activation function                     & GELU \\
Value layer normalization                     & True \\
Value ensemble size                            & 2 \\
Value ensemble operator                       & MEAN \\
Chunking horizon                              & 5 (1 for\textit{ FQL, BFN}) \\
\midrule
\multicolumn{2}{l}{\textcolor{our_color}{\textbf{\textit{DFP (Ours)}}}} \\
\addlinespace[2pt]
Number of candidates for top-$K$ ($N$)                          & 16 \\
Number of generated actions ($N_{\text{gen}}$)                          & 8 \\
Drift weight ($\lambda$)                        & 0.5 \\
Actor EMA smoothing ($\tau_{\mathrm{EMA}}$)   & $1 \times 10^{-4}$ \\
Number of best-of-$N'$ ($N'$)                     & 16$^*$ / 4$^{\dagger}$ \\
Top-$K$ positive set size ($K$)                    & 4$^*$ / 2$^{\dagger}$ \\
Kernel bandwidth ($h$)                     & $\{0.05\}^*$ / $\{0.005, 0.05\}$$^{\dagger}$ \\
\bottomrule
\end{tabular}

\vspace{4pt}
{\small  $^*$ OGBench (\texttt{cube}) / $^{\dagger}$ Robomimic (\texttt{lift}, \texttt{can}, \texttt{square}).}
\end{table}
\section{Additional Results}

\subsection{Full Offline-to-Online Results}
\label{app:full-offline-to-online}

\definecolor{kaistblue}{RGB}{0,65,145}
\definecolor{kaistlight}{RGB}{210,225,245}
\definecolor{onlinecol}{RGB}{0,110,170}

\definecolor{mutedgray}{gray}{0.45}
\definecolor{mutedgrayLite}{gray}{0.7}
\definecolor{stripegray}{gray}{0.96}

\definecolor{darkemph}{gray}{0.20}

\definecolor{ourscol}{RGB}{225,235,250}
\definecolor{ourscolstrong}{RGB}{200,220,245}

\newcommand{\stdsup}[2]{%
  \textsuperscript{\tiny\textcolor{#1}{$\pm$#2}}%
}

\newcommand{\stdsub}[2]{%
  \textsubscript{\tiny\textcolor{#1}{$\pm$#2}}%
}

\newcommand{\trans}[4]{%
  \textcolor{mutedgray}{#1}\stdsub{mutedgrayLite}{#2}%
  {\scriptsize$\rightarrow$}%
  \textcolor{mutedgray}{#3}\stdsub{mutedgrayLite}{#4}%
}

\newcommand{\transBB}[4]{
  \textcolor{darkemph}{\textbf{#1}}\stdsub{mutedgrayLite}{#2}%
  {\scriptsize$\rightarrow$}%
  \textcolor{darkemph}{\textbf{#3}}\stdsub{mutedgrayLite}{#4}%
}
\newcommand{\transBS}[4]{%
  \textcolor{darkemph}{\textbf{#1}}\stdsub{mutedgrayLite}{#2}%
  {\scriptsize$\rightarrow$}%
  \textcolor{darkemph}{\underline{#3}}\stdsub{mutedgrayLite}{#4}%
}
\newcommand{\transBN}[4]{%
  \textcolor{darkemph}{\textbf{#1}}\stdsub{mutedgrayLite}{#2}%
  {\scriptsize$\rightarrow$}%
  \textcolor{mutedgray}{#3}\stdsub{mutedgrayLite}{#4}%
}
\newcommand{\transSB}[4]{
  \textcolor{darkemph}{\underline{#1}}\stdsub{mutedgrayLite}{#2}%
  {\scriptsize$\rightarrow$}%
  \textcolor{darkemph}{\textbf{#3}}\stdsub{mutedgrayLite}{#4}%
}
\newcommand{\transSS}[4]{%
  \textcolor{darkemph}{\underline{#1}}\stdsub{mutedgrayLite}{#2}%
  {\scriptsize$\rightarrow$}%
  \textcolor{darkemph}{\underline{#3}}\stdsub{mutedgrayLite}{#4}%
}
\newcommand{\transSN}[4]{%
  \textcolor{darkemph}{\underline{#1}}\stdsub{mutedgrayLite}{#2}%
  {\scriptsize$\rightarrow$}%
  \textcolor{mutedgray}{#3}\stdsub{mutedgrayLite}{#4}%
}
\newcommand{\transNB}[4]{%
  \textcolor{mutedgray}{#1}\stdsub{mutedgrayLite}{#2}%
  {\scriptsize$\rightarrow$}%
  \textcolor{darkemph}{\textbf{#3}}\stdsub{mutedgrayLite}{#4}%
}
\newcommand{\transNS}[4]{%
  \textcolor{mutedgray}{#1}\stdsub{mutedgrayLite}{#2}%
  {\scriptsize$\rightarrow$}%
  \textcolor{darkemph}{\underline{#3}}\stdsub{mutedgrayLite}{#4}%
}

\newcommand{\transOBB}[4]{%
  \textcolor{kaistblue}{\textbf{#1}}\stdsub{kaistblue}{#2}%
  {\scriptsize$\rightarrow$}%
  \textcolor{kaistblue}{\textbf{#3}}\stdsub{kaistblue}{#4}%
}
\newcommand{\transOBS}[4]{%
  \textcolor{kaistblue}{\textbf{#1}}\stdsub{kaistblue}{#2}%
  {\scriptsize$\rightarrow$}%
  \textcolor{kaistblue}{\underline{#3}}\stdsub{kaistblue}{#4}%
}
\newcommand{\transOBN}[4]{%
  \textcolor{kaistblue}{\textbf{#1}}\stdsub{kaistblue}{#2}%
  {\scriptsize$\rightarrow$}%
  \textcolor{mutedgray}{#3}\stdsub{mutedgrayLite}{#4}%
}
\newcommand{\transOSB}[4]{%
  \textcolor{kaistblue}{\underline{#1}}\stdsub{kaistblue}{#2}%
  {\scriptsize$\rightarrow$}%
  \textcolor{kaistblue}{\textbf{#3}}\stdsub{kaistblue}{#4}%
}
\newcommand{\transOSS}[4]{%
  \textcolor{kaistblue}{\underline{#1}}\stdsub{kaistblue}{#2}%
  {\scriptsize$\rightarrow$}%
  \textcolor{kaistblue}{\underline{#3}}\stdsub{kaistblue}{#4}%
}
\newcommand{\transOSN}[4]{%
  \textcolor{kaistblue}{\underline{#1}}\stdsub{kaistblue}{#2}%
  {\scriptsize$\rightarrow$}%
  \textcolor{mutedgray}{#3}\stdsub{mutedgrayLite}{#4}%
}
\newcommand{\transONB}[4]{%
  \textcolor{mutedgray}{#1}\stdsub{mutedgrayLite}{#2}%
  {\scriptsize$\rightarrow$}%
  \textcolor{kaistblue}{\textbf{#3}}\stdsub{kaistblue}{#4}%
}
\newcommand{\transONS}[4]{%
  \textcolor{mutedgray}{#1}\stdsub{mutedgrayLite}{#2}%
  {\scriptsize$\rightarrow$}%
  \textcolor{kaistblue}{\underline{#3}}\stdsub{kaistblue}{#4}%
}
\newcommand{\transONN}[4]{%
  \textcolor{mutedgray}{#1}\stdsub{mutedgrayLite}{#2}%
  {\scriptsize$\rightarrow$}%
  \textcolor{mutedgray}{#3}\stdsub{mutedgrayLite}{#4}%
}


\begin{table*}[t]
\centering
\caption{
\textbf{Offline-to-online RL full results.}
Each cell shows offline $\rightarrow$ online (mean with std as subscript). 
Best result per column in \textbf{bold}; second-best \underline{underlined}.}
\label{tab:offline-to-online}

\setlength{\tabcolsep}{4pt}
\renewcommand{\arraystretch}{1.3}
\footnotesize

\resizebox{\textwidth}{!}{%
\begin{tabular}{@{}ll|cccccc|>{\columncolor{ourscol}}c|>{\columncolor{ourscol}}c@{}}
\toprule

\multicolumn{2}{c|}{} 
& \multicolumn{6}{c|}{\textbf{Baselines}} 
& \multicolumn{2}{c}{\cellcolor{ourscol}{\textcolor{our_color}{\textbf{Ours}}}} \\

\cmidrule(lr){3-8} \cmidrule(lr){9-10}

\textbf{Benchmark} & \textbf{Task}
& BFN~\cite{ghasemipour2021:EMAQ} & QC-BFN~\cite{Li2025:QC} & FQL~\cite{Park2025:FQL} & QC-FQL~\cite{Li2025:QC} & MVP~\cite{Zhan2026:MVP} & MVP w/ $\mathcal{L}_{\text{top-}K}$
& \textbf{DFP w/o $\mathcal{L}_{\text{top-}K}$} & \textcolor{our_color}{\textbf{DFP (Ours)}} \\

\midrule

\multirow{3}{*}{\textbf{Robomimic}}
& lift
& \trans{90.4}{7.1}{97.6}{2.2}
& \transSN{95.2}{3.4}{99.6}{0.5}
& \trans{84.0}{5.7}{96.8}{2.3}
& \transBB{95.8}{3.3}{100.0}{0.0}
& \trans{61.4}{5.7}{99.8}{0.4}
& \trans{61.4}{5.7}{100.0}{0.0}
& \transONB{85.2}{3.4}{100.0}{0.0}
& \transONB{85.2}{3.4}{100.0}{0.0} \\

& square
& \trans{16.8}{5.2}{32.8}{7.6}
& \transSN{40.0}{2.0}{88.4}{3.6}
& \trans{3.6}{3.6}{10.8}{6.7}
& \trans{35.4}{6.7}{72.0}{8.7}
& \trans{4.6}{2.5}{79.4}{3.9}
& \trans{4.6}{2.5}{81.6}{5.0}
& \transOBS{77.6}{6.9}{88.6}{1.5}
& \transOBB{77.6}{6.9}{93.2}{1.6} \\

& can
& \trans{59.6}{12.8}{82.0}{2.4}
& \transSN{83.0}{2.7}{90.6}{3.2}
& \trans{31.2}{4.1}{58.4}{7.5}
& \transBB{88.0}{3.3}{94.4}{1.8}
& \trans{46.0}{7.1}{83.6}{5.2}
& \trans{46.0}{7.1}{86.2}{5.6}
& \transONN{25.8}{4.1}{90.4}{4.5}
& \transONS{25.8}{4.1}{90.6}{3.2} \\

\midrule

\multirow{3}{*}{\textbf{Cube-double}}
& task2
& \transSN{75.6}{7.4}{86.0}{4.7}
& \transBS{77.6}{4.7}{99.8}{0.4}
& \trans{27.2}{10.4}{93.2}{7.8}
& \transNB{39.6}{9.7}{100.0}{0.0}
& \trans{34.6}{9.8}{98.4}{1.3}
& \trans{34.6}{9.8}{99.6}{0.5}
& \transONN{53.2}{5.8}{99.2}{0.8}
& \transONB{53.2}{5.8}{100.0}{0.0} \\

& task3
& \transBN{79.6}{6.8}{88.8}{5.2}
& \transSB{76.0}{4.7}{99.8}{0.4}
& \trans{26.8}{9.3}{91.2}{4.8}
& \transNB{40.2}{6.2}{99.8}{0.4}
& \trans{37.8}{10.8}{98.6}{1.1}
& \trans{37.8}{10.8}{98.8}{0.8}
& \transONS{58.8}{5.0}{99.6}{0.9}
& \transONS{58.8}{5.0}{99.6}{0.5} \\

& task4
& \transSN{18.4}{5.0}{27.2}{8.2}
& \transBN{23.2}{5.8}{92.6}{5.7}
& \trans{4.0}{3.2}{6.0}{6.3}
& \transNB{9.4}{3.6}{99.8}{0.4}
& \trans{15.0}{7.0}{94.8}{4.3}
& \trans{15.0}{7.0}{96.6}{0.5}
& \transONN{8.8}{2.6}{96.0}{2.8}
& \transONS{8.8}{2.6}{99.6}{0.5} \\

\midrule

\multirow{3}{*}{\textbf{Cube-triple}}
& task2
& \trans{0.4}{0.9}{7.6}{8.5}
& \transSN{0.6}{0.9}{87.4}{9.8}
& \trans{0.4}{0.9}{0.4}{0.9}
& \trans{0.0}{0.0}{88.2}{2.2}
& \trans{0.4}{0.5}{86.2}{4.4}
& \trans{0.4}{0.5}{78.0}{15.4}
& \transOBS{3.2}{2.5}{91.4}{3.4}
& \transOBB{3.2}{2.5}{98.4}{1.1} \\

& task3
& \trans{2.0}{2.0}{6.8}{3.0}
& \trans{0.8}{0.8}{80.8}{3.8}
& \trans{0.4}{0.9}{6.4}{8.2}
& \trans{0.0}{0.0}{60.4}{12.3}
& \transSN{4.2}{2.9}{57.2}{10.5}
& \transSN{4.2}{2.9}{60.6}{12.6}
& \transOBS{7.6}{3.5}{83.2}{4.1}
& \transOBB{7.6}{3.5}{91.6}{1.8} \\

& task4
& \trans{0.0}{0.0}{0.0}{0.0}
& \trans{0.0}{0.0}{33.4}{9.4}
& \trans{0.0}{0.0}{0.0}{0.0}
& \transNS{0.0}{0.0}{51.4}{24.2}
& \transBN{1.2}{0.8}{31.0}{20.3}
& \transBN{1.2}{0.8}{30.4}{10.9}
& \transOSN{1.0}{0.7}{31.2}{6.5}
& \transOSB{1.0}{0.7}{81.2}{5.6} \\

\midrule

\multirow{3}{*}{\textbf{Cube-quad}}
& task2
& \trans{0.0}{0.0}{32.4}{21.0}
& \trans{0.0}{0.0}{95.8}{2.3}
& \trans{0.0}{0.0}{0.0}{0.0}
& \trans{0.0}{0.0}{98.0}{2.0}
& \trans{0.0}{0.0}{96.6}{1.5}
& \trans{0.0}{0.0}{97.2}{1.5}
& \transOBS{0.2}{0.4}{97.6}{1.3}
& \transOBB{0.2}{0.4}{99.6}{0.9} \\

& task3
& \trans{0.0}{0.0}{0.0}{0.0}
& \transSN{1.6}{1.8}{63.2}{9.7}
& \trans{0.0}{0.0}{0.0}{0.0}
& \trans{0.0}{0.0}{85.0}{6.9}
& \trans{1.2}{2.2}{47.2}{29.8}
& \trans{1.2}{2.2}{69.2}{14.8}
& \transOBS{5.6}{2.5}{88.8}{3.2}
& \transOBB{5.6}{2.5}{96.6}{1.9} \\

& task4
& \trans{0.0}{0.0}{0.0}{0.0}
& \transBN{0.4}{0.5}{74.2}{10.9}
& \trans{0.0}{0.0}{0.0}{0.0}
& \trans{0.0}{0.0}{92.2}{7.0}
& \trans{0.0}{0.0}{91.2}{2.4}
& \trans{0.0}{0.0}{94.2}{4.1}
& \transOSS{0.2}{0.4}{95.2}{3.2}
& \transOSB{0.2}{0.4}{99.0}{1.7} \\

\midrule

\multicolumn{2}{c|}{\textbf{Average}}
& \textcolor{mutedgray}{28.6}{\scriptsize$\rightarrow$}\textcolor{mutedgray}{38.4}
& \textcolor{darkemph}{\textbf{33.2}}{\scriptsize$\rightarrow$}\textcolor{mutedgray}{83.8}
& \textcolor{mutedgray}{14.8}{\scriptsize$\rightarrow$}\textcolor{mutedgray}{30.3}
& \textcolor{mutedgray}{25.7}{\scriptsize$\rightarrow$}\textcolor{mutedgray}{86.8}
& \textcolor{mutedgray}{17.2}{\scriptsize$\rightarrow$}\textcolor{mutedgray}{80.3}
& \textcolor{mutedgray}{17.2}{\scriptsize$\rightarrow$}\textcolor{mutedgray}{82.7}
& \textcolor{kaistblue}{\underline{27.3}}{\scriptsize$\rightarrow$}\textcolor{kaistblue}{\underline{88.4}}
& \textcolor{kaistblue}{\underline{27.3}}{\scriptsize$\rightarrow$}\textcolor{kaistblue}{\textbf{95.8}} \\
\bottomrule
\end{tabular}
}
\end{table*}
Tab.~\ref{tab:offline-to-online} reports the full offline$\rightarrow$online progression of all methods, complementing the online-only summary in Tab.~\ref{tab:main_results}. The expanded format makes the online-phase contribution of each method explicit, isolating how effectively each parameterization absorbs buffer updates during fine-tuning. Consistent with the analysis in Sec.~\ref{sec:results}, \textbf{DFP} and \textbf{DFP w/o $\mathcal{L}_{\text{top-}K}$} show the largest online gains on the long-horizon cube-triple and cube-quadruple splits, where the drifting backbone's output-level supervision and the top-$K$ drift loss provide the most leverage.

\subsection{Adaptation to Offline RL}
\label{app:offline_rl}


While our main experiments in Sec.~\ref{sec:results} follow the offline-to-online RL setup of MVP~\cite{Zhan2026:MVP}, DFP's algorithm is not tied to this setting. To probe the generality of our method, we evaluate DFP in a pure offline RL setting, where the algorithm is identical to Algorithm~\ref{alg:dfp} but without the online environment interaction. The combined loss  $\mathcal{L}_{\mathrm{BC}}(\theta) + \lambda\, \mathcal{L}_{\text{top-}K}(\theta)$ is applied throughout training.


\paragraph{Offline RL Benchmarks.}
Following prior offline RL work~\cite{Park2025:FQL}, we evaluate on the $5$ robot-manipulation environments from OGBench~\cite{park2024:OGbench}, and compare against the baselines reported therein. For each environment, we run the default tasks and report success rates averaged over $8$ seeds. Baseline numbers are taken from prior work~\cite{Park2025:FQL}.

\paragraph{Offline RL Baselines.} 
We compare against three categories of baselines: (i) \textit{Gaussian policy}: Behavior Cloning (BC), Implicit Q-Learning (IQL)~\cite{kostrikov2022:IQL}, and ReBRAC~\cite{tarasov2023:ReBRAC}; (ii) \textit{diffusion-based policy}: Implicit Diffusion Q-Learning (IDQL)~\cite{hansen2023:IDQL}, SRPO~\cite{chen2024:SRPO}, and CAC~\cite{Ding2024:consistencypolicy}; (iii) \textit{flow-based policy}: FAWAC and FBRAC, the flow-based variants of AWAC~\cite{Nair2020:AWAC} and BRAC~\cite{Wu2019:BRAC} respectively, and Flow Q-Learning with its iterative form (FQL, IFQL)~\cite{Park2025:FQL}.


\paragraph{Offline RL Results.}
Although our primary target is offline-to-online fine-tuning, DFP can also be trained from scratch in the pure offline RL setting by jointly optimizing behavior cloning and the top-$K$ Q-maximization objective. As Tab.~\ref{tab:offline_default} shows, DFP attains the best success rate on \texttt{cube-double-task2} and \texttt{scene-task2} and remains close to the strongest baseline on \texttt{cube-single-task2} and \texttt{puzzle-4x4-task4}, indicating effectiveness in this setting as well.

\begin{table}[h]
\centering
\caption{\textbf{Offline RL results on default OGBench tasks.} Mean with std as subscript, over 8 seeds. Best result per task in \textbf{bold}; second-best \underline{underlined}.}
\label{tab:offline_default}
\setlength{\tabcolsep}{4pt}
\renewcommand{\arraystretch}{1.15}
\scriptsize
\resizebox{\textwidth}{!}{%
\begin{tabular}{l|ccc|ccc|cccc|c}
\toprule
& \multicolumn{3}{c|}{\textbf{Gaussian Policies}} 
& \multicolumn{3}{c|}{\textbf{Diffusion Policies}} 
& \multicolumn{4}{c|}{\textbf{Flow Policies}} 
& \textcolor{our_color}{\textbf{Ours}} \\
\cmidrule(lr){2-4} \cmidrule(lr){5-7} \cmidrule(lr){8-11} \cmidrule(lr){12-12}
\textbf{Task}
& BC & IQL & ReBRAC
& IDQL & SRPO & CAC
& FAWAC & FBRAC & IFQL & FQL
& \textcolor{our_color}{\textbf{DFP}} \\
\midrule
cube-single-task2 
& \val{3}{1} & \val{85}{8} & \val{92}{4}
& \underline{\val{96}{2}} & \val{82}{16} & \val{80}{30}
& \val{81}{9} & \val{83}{13} & \val{73}{3}
& \textbf{\val{97}{2}}
& \val{95}{3} \\
cube-double-task2 
& \val{0}{0} & \val{1}{1} & \val{7}{3}
& \val{16}{10} & \val{0}{0} & \val{2}{2}
& \val{2}{1} & \val{22}{12} & \val{9}{5}
& \underline{\val{36}{6}}
& \textbf{\val{41}{4}} \\
scene-task2 
& \val{1}{1} & \val{12}{3} & \val{50}{13}
& \val{33}{14} & \val{2}{2} & \val{50}{40}
& \val{18}{8} & \val{46}{10} & \val{0}{0}
& \underline{\val{76}{9}}
& \textbf{\val{93}{4}} \\
puzzle-3x3-task4 
& \val{1}{1} & \val{2}{1} & \val{2}{1}
& \val{0}{0} & \val{0}{0} & \val{0}{0}
& \val{1}{1} & \val{2}{2} & \val{0}{0}
& \textbf{\val{16}{5}}
& \underline{\val{3}{2}} \\
puzzle-4x4-task4 
& \val{0}{0} & \val{4}{1} & \val{10}{3}
& \textbf{\val{26}{6}} & \val{7}{4} & \val{1}{1}
& \val{0}{0} & \val{5}{1} & \underline{\val{21}{11}}
& \val{11}{3}
& \val{20}{2} \\
\bottomrule
\end{tabular}
}
\end{table}

\paragraph{Hyperparameters.}
We list the offline-specific hyperparameters in Tab.~\ref{tab:offline_hyperparams}. Unlike the offline-to-online setting, we disable both action chunking and best-of-$N'$ execution for offline RL. All remaining hyperparameters follow the shared block of Tab.~\ref{tab:hyperparameters}
\begin{table}[h]
\centering
\caption{\textbf{Task-specific hyperparameters for DFP offline training.}}
\label{tab:offline_hyperparams}
\setlength{\tabcolsep}{6pt}
\renewcommand{\arraystretch}{1.15}
\scriptsize
\begin{tabular}{l|ccccc}
\toprule
\textbf{Hyperparameter}
& \makecell{cube-single\\task2}
& \makecell{cube-double\\task2}
& \makecell{scene\\task2}
& \makecell{puzzle-3x3\\task4}
& \makecell{puzzle-4x4\\task4} \\
\midrule
Drift weight $\lambda$                        &  0.45 & 0.55 & 0.5 & 0.5 & 0.5  \\
Kernel bandwidth ($h$)                            & $\{0.05\}$ &  $\{0.05\}$ & $\{0.01, 0.05\}$ & $\{0.01, 0.05\}$ & $\{0.01, 0.05\}$  \\
Num. candidates for top-$K$ ($N$)           & 16 & 16 & 16 & 16 & 32  \\
Top-$K$ positive set size ($K$)       & 8 & 8 & 8 & 8 & 16  \\
\bottomrule
\end{tabular}
\end{table}

\subsection{Asymmetric Effect of \texorpdfstring{$\mathcal{L}_{\text{top-}K}$}{L\_K} Across Backbones}
\label{app:lk-asymmetry}

Figure~\ref{fig:lk_asymmetry} shows the training curves of the backbone $\times$ loss ablation. Adding $\mathcal{L}_{\text{top-}K}$ to the MeanFlow backbone (\textbf{MVP} $\rightarrow$ \textbf{MVP w/ $\mathcal{L}_{\text{top-}K}$}) yields only a marginal gain that often falls within the seed variance, while the same supervision on the drifting backbone (\textbf{DFP w/o $\mathcal{L}_{\text{top-}K}$} $\rightarrow$ \textbf{DFP}) produces a substantially larger lift, most pronounced on the long-horizon \texttt{cube-triple} and \texttt{cube-quadruple} splits.

\begin{figure}[t]
\centering
\captionsetup[subfigure]{justification=centering, aboveskip=1pt, belowskip=1pt, labelformat=empty}
\newcommand{\rotlabelb}[2][1.5]{\raisebox{#1\height}{\rotatebox[origin=c]{90}{\scriptsize #2}}}
\rotlabelb[1.7]{Cube-triple-task2}\hspace{1pt}%
\subfloat[\label{subFig:lkabl_triple-2}]
{\includegraphics[width=0.31\textwidth,trim={0.7cm 0.2cm 0.0cm 0.2cm}, clip]{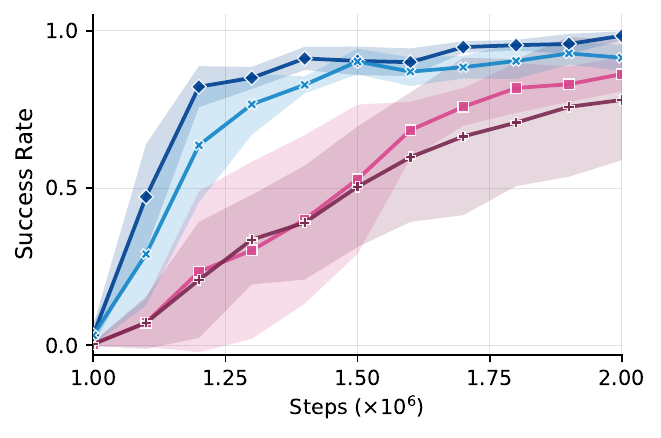}}\hspace{2pt}%
\rotlabelb[1.7]{Cube-triple-task3}\hspace{1pt}%
\subfloat[\label{subFig:lkabl_triple-3}]
{\includegraphics[width=0.31\textwidth,trim={0.7cm 0.2cm 0.0cm 0.2cm}, clip]{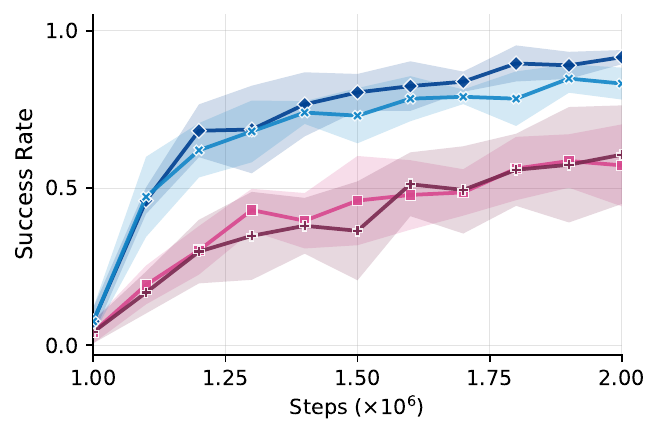}}\hspace{2pt}%
\rotlabelb[1.7]{Cube-triple-task4}\hspace{1pt}%
\subfloat[\label{subFig:lkabl_triple-4}]
{\includegraphics[width=0.31\textwidth,trim={0.7cm 0.2cm 0.0cm 0.2cm}, clip]{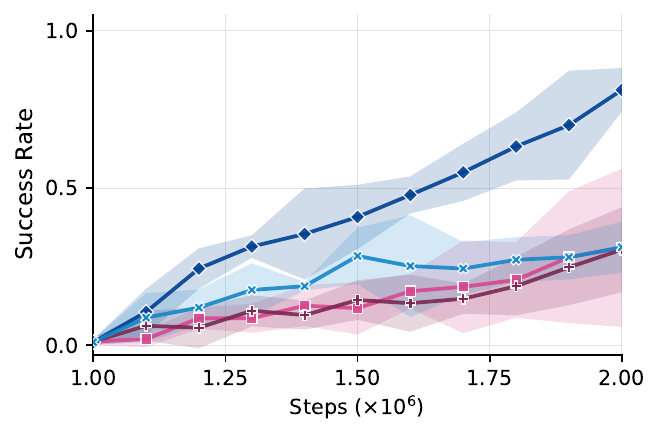}} \\[-1.5em]
{\includegraphics[width=0.6\textwidth,trim={0.0cm 0.2cm 0.0cm 0.2cm}, clip]{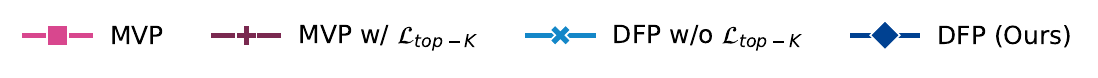}}
\vspace{-0.1cm}
\caption{\textbf{Online training curves for the backbone $\times$ loss ablation.} Success rate over the online phase, comparing MVP~\cite{Zhan2026:MVP}, MVP w/ $\mathcal{L}_{\text{top-}K}$, DFP w/o $\mathcal{L}_{\text{top-}K}$, and DFP.}
\label{fig:lk_asymmetry}
\end{figure}

\subsection{Top-\texorpdfstring{$K$}{K} Positive Set Size}
\label{app:k-ablation-curves}

Fig.~\ref{fig:k_ablation} shows the training curves for $K \in \{1, 2, 4, 8\}$ in $\mathcal{L}_{\text{top-}K}$, with all other components fixed.

\begin{figure}[t]
\centering
\captionsetup[subfigure]{justification=centering, aboveskip=4pt, belowskip=4pt}
\subfloat[Robomimic-lift\label{subFig:kabl_lift}]
{\includegraphics[width = 0.32\textwidth]{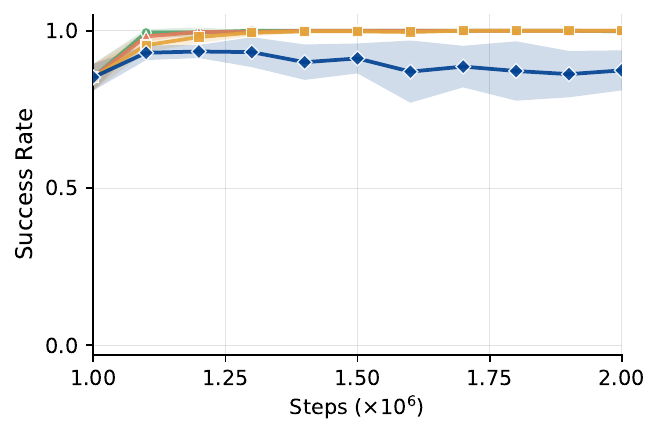}} 
\subfloat[Robomimic-square\label{subFig:kabl_square}]
{\includegraphics[width = 0.32\textwidth]{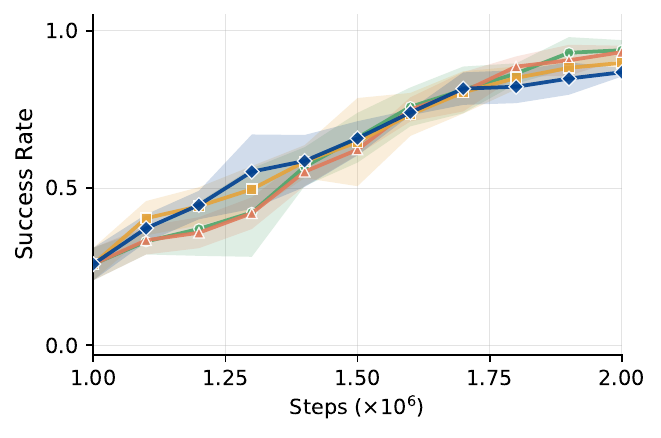}}
\subfloat[Robomimic-can\label{subFig:kabl_can}]
{\includegraphics[width = 0.32\textwidth]{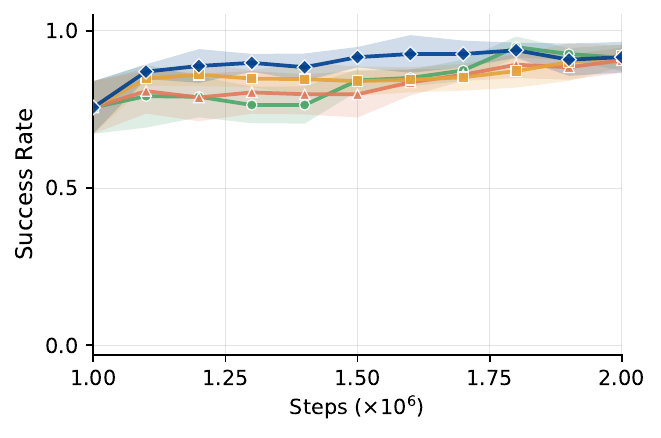}} \\[-0.5em]
\subfloat[Cube-double-task2\label{subFig:kabl_double-2}]
{\includegraphics[width = 0.32\textwidth]{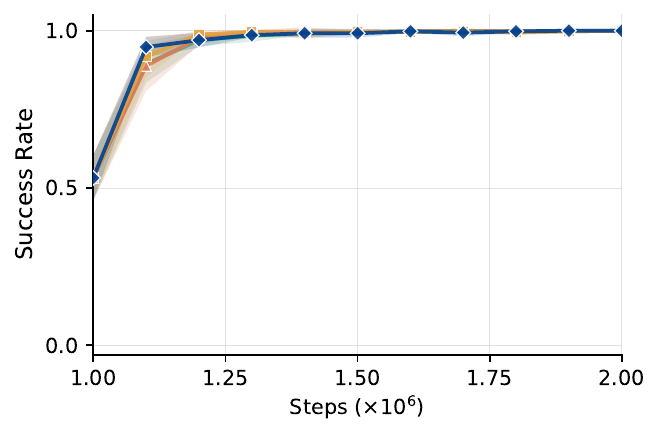}} 
\subfloat[Cube-double-task3\label{subFig:kabl_double-3}]
{\includegraphics[width = 0.32\textwidth]{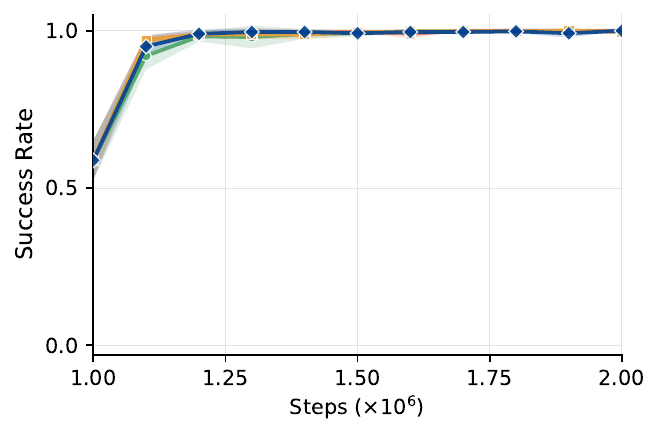}} 
\subfloat[Cube-double-task4\label{subFig:kabl_double-4}]
{\includegraphics[width = 0.32\textwidth]{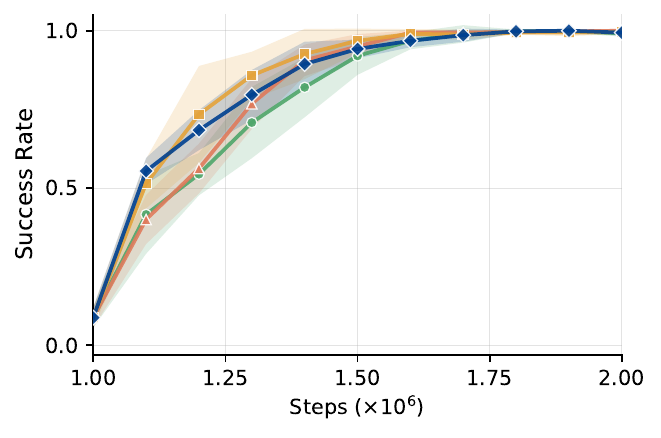}} \\[-0.5em]
\subfloat[Cube-triple-task2\label{subFig:kabl_triple-2}]
{\includegraphics[width = 0.32\textwidth]{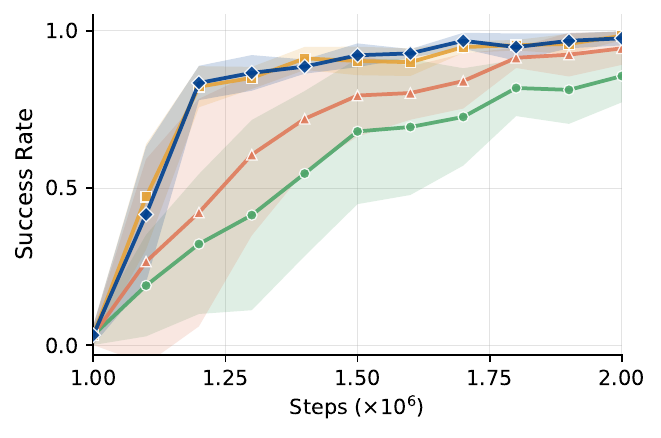}} 
\subfloat[Cube-triple-task3\label{subFig:kabl_triple-3}]
{\includegraphics[width = 0.32\textwidth]{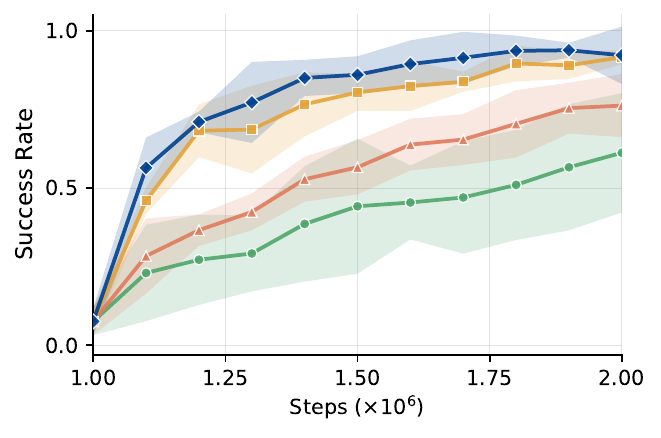}}
\subfloat[Cube-triple-task4\label{subFig:kabl_triple-4}]
{\includegraphics[width = 0.32\textwidth]{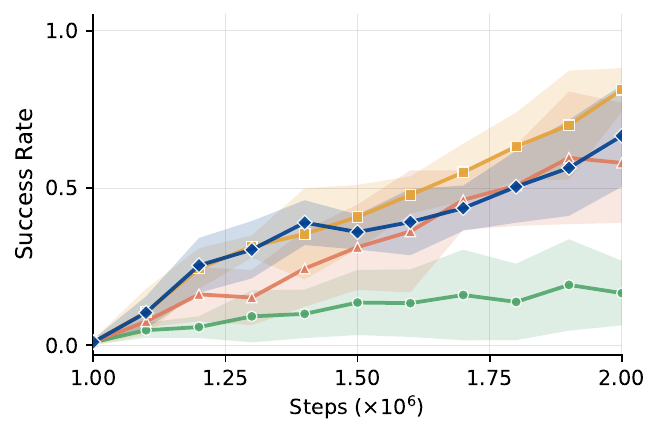}} \\[-0.5em]
\subfloat[Cube-quadruple-task2\label{subFig:kabl_quad-2}]
{\includegraphics[width = 0.32\textwidth]{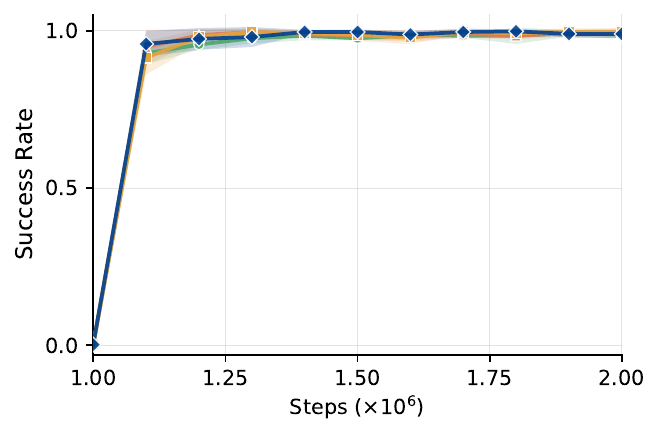}} 
\subfloat[Cube-quadruple-task3\label{subFig:kabl_quad-3}]
{\includegraphics[width = 0.32\textwidth]{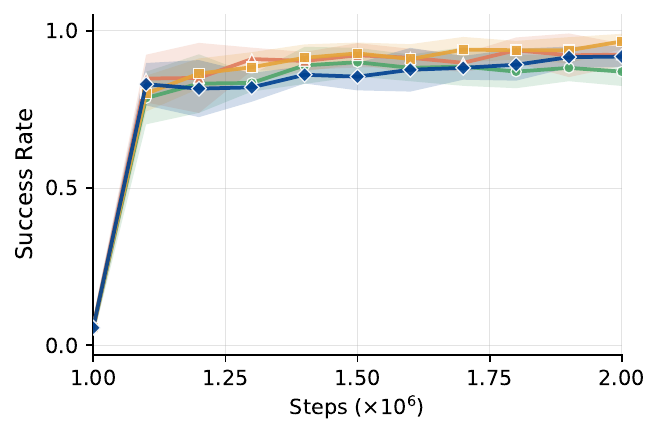}}
\subfloat[Cube-quadruple-task4\label{subFig:kabl_quad-4}]
{\includegraphics[width = 0.32\textwidth]{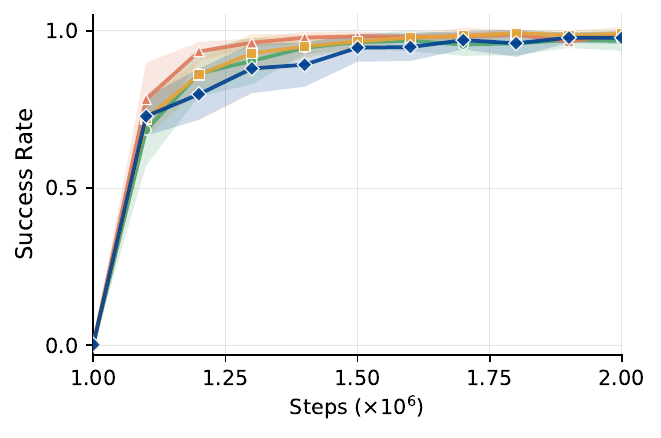}} \\[0.2em]
{\includegraphics[width = 0.5\textwidth, clip]{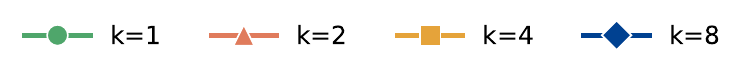}}
\caption{\textbf{Online training curves for the top-$K$ ablation.} Success rate over the online phase on Robomimic (top row) and the OGBench.}
\label{fig:k_ablation}
\end{figure}

\subsection{~\texorpdfstring{$\lambda$}{} Ablations}
\label{app:lambda_ablation}

Fig.~\ref{fig:lambda_sweep} shows the effect of the drift weight $\lambda \in \{0.1, 0.5, 1.0, 5.0\}$ on the cube-quadruple tasks. DFP is robust for $\lambda \geq 0.5$, with all settings yielding near-identical performance.

\begin{figure}[t]
\centering
\captionsetup[subfigure]{justification=centering, aboveskip=1pt, belowskip=1pt}
\subfloat[Cube-quadruple-task2\label{subFig:lkabl_triple-2}]
{\includegraphics[width = 0.32\textwidth]{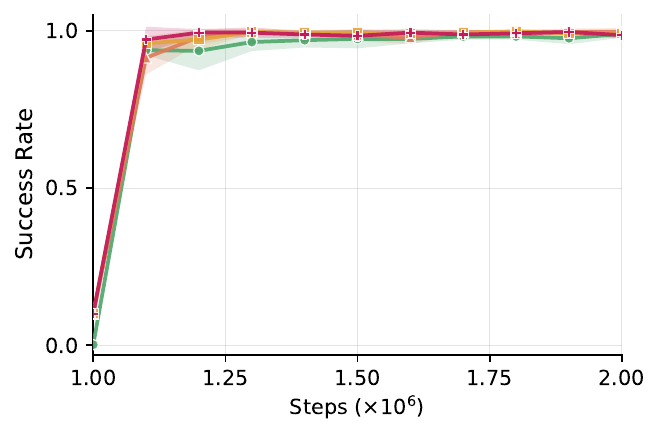}} 
\subfloat[Cube-quadruple-task3\label{subFig:lkabl_triple-3}]
{\includegraphics[width = 0.32\textwidth]{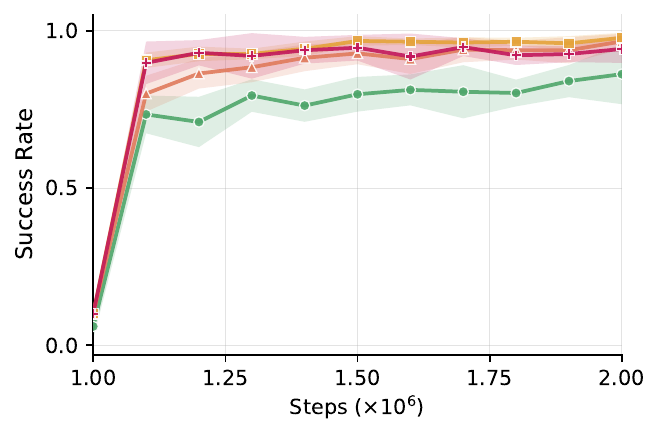}}
\subfloat[Cube-quadruple-task4\label{subFig:lkabl_triple-4}]
{\includegraphics[width = 0.32\textwidth]{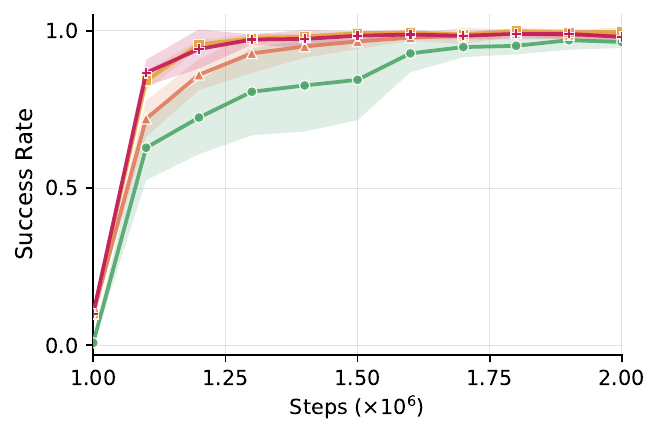}} \\
{\includegraphics[width = 0.6\textwidth, clip]
{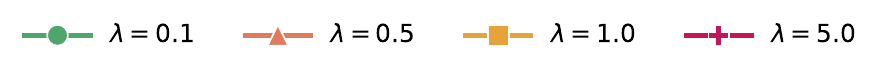}}
\caption{\textbf{Online training curves for different $\lambda$ values.} Success rate over the online phase.}
\label{fig:lambda_sweep}
\end{figure}

\subsection{Training and Inference Cost}
\label{app:timecost_analysis}

We measure per-step wall-clock cost for online training and inference, averaged across the $12$ benchmark tasks (Tab.~\ref{tab:cost_comparison}); inference cost is measured on CPU following the protocol of MVP~\cite{Zhan2026:MVP}. In our implementation, the measured overhead of $\mathcal{L}_{\text{top-}K}$ is small: DFP (13.60 ms/step) is within 0.04 ms of DFP w/o $\mathcal{L}_{\text{top-}K}$, and the full DFP cost is comparable to other one-step baselines and faster than the multi-step BFN and QC-BFN. At inference, DFP requires a single forward pass of $f_\theta$ and stays in the same regime as the other one-step baselines (FQL, QC-FQL, MVP), all roughly an order of magnitude faster than BFN and QC-BFN. The gains reported in Sec.~\ref{sec:results} thus come at negligible training overhead and at the inference cost of a standard one-step policy.

\begin{table}[t!]
\centering
\scriptsize
\caption{Inference and online training cost comparison across baselines. We report the mean wall-clock time (ms) per online training step and per inference step, averaged over 12 tasks.}
\label{tab:cost_comparison}
\setlength{\tabcolsep}{3pt}
\renewcommand{\arraystretch}{1.15}
\resizebox{\textwidth}{!}{%
\begin{tabular}{@{}lcccccccc@{}}
\toprule
& \textbf{BFN} & \textbf{QC-BFN} & \textbf{FQL} & \textbf{QC-FQL} & \textbf{MVP} & \textbf{MVP w/ $\mathcal{L}_{\text{top-}K}$} & \textbf{DFP w/o $\mathcal{L}_{\text{top-}K}$} & \textcolor{our_color}{\textbf{DFP (Ours)}} \\
\midrule
\textbf{Online Cost (ms)} & 15.30 & 16.09 & 12.00 & 12.56 & 13.20 & 13.48 & 13.56 & 13.60 \\
\textbf{Evaluation Cost (ms)} & 291.33 & 295.07 & 21.90 & 21.93 & 25.71 & 25.71 & 29.27 & 29.27 \\

\bottomrule
\end{tabular}
}
\end{table}
\clearpage
\section{Proofs}
\label{app:proofs}

We provide the formal derivations of the equations in Sec.~\ref{sec:wgf_theory} in Appendix~\ref{app:wgf_derivations} and the proof of Proposition~\ref{prop:topk_limit} in Appendix~\ref{app:topk_limit_proof}.

\subsection{Derivations for Sec.~\ref{sec:wgf_theory}}
\label{app:wgf_derivations}

\paragraph{Lagrangian derivation of $\pi^+$ (Eq.~\eqref{eq:pi_plus_objective}).} The closed form follows from the standard Lagrangian derivation of the KL-regularized $Q$-maximization update~\cite{Levine2018:RLasInference, Haarnoja2018:SAC}; we leave the derivation here for completeness. The per-iteration regularized optimal policy is defined as the maximizer of the KL-regularized policy improvement objective,
\begin{equation}
\pi^+(\cdot | s) = \arg\max_{\pi} \;\; \mathbb{E}_{a \sim \pi}[Q_\phi(s, a)] \;-\; \alpha\, \KL\bigl(\pi(\cdot | s)\, \big\|\, \pi_{\mathrm{old}}(\cdot | s)\bigr) \quad \text{s.t. } \int \pi(a | s)\, da = 1.
\end{equation}
Forming the Lagrangian (state $s$ fixed, with multiplier $\lambda$ for normalization),
\begin{equation}
\mathcal{L}(\pi, \lambda) = \int \pi(a)\, Q_\phi(s, a)\, da - \alpha \int \pi(a)\, \log\frac{\pi(a)}{\pi_{\mathrm{old}}(a)}\, da + \lambda\Bigl(1 - \int \pi(a)\, da\Bigr),
\end{equation}
and setting $\delta \mathcal{L}/\delta \pi(a) = 0$,
\begin{equation}
Q_\phi(s, a) - \alpha \Bigl(\log\frac{\pi(a)}{\pi_{\mathrm{old}}(a)} + 1\Bigr) - \lambda = 0.
\end{equation}
Solving yields $\pi(a) = \pi_{\mathrm{old}}(a)\, \exp\bigl((Q_\phi(s, a) - \alpha - \lambda)/\alpha\bigr)$, and absorbing the $a$-independent factor into the partition function,
\begin{equation}
\pi^+(a | s) = \frac{\pi_{\mathrm{old}}(a | s)\, \exp(Q_\phi(s, a)/\alpha)}{Z(s)}, \qquad Z(s) = \int \pi_{\mathrm{old}}(a' | s)\, \exp(Q_\phi(s, a')/\alpha)\, da'. \qed
\end{equation}

\begin{lemma}[$W_2$ gradient flow velocity of $\KL(q \| p)$]
\label{lem:wgf_kl_velocity}
For a fixed reference $p \in \mathcal{P}_2(\mathbb{R}^d)$ with absolutely continuous, strictly positive density, the $W_2$ gradient flow particle velocity of the KL functional $\mathcal{F}(q) = \KL(q \| p)$ is the score difference $v_t(x) = \nabla_x \log p(x) - \nabla_x \log q_t(x)$ (Eq.~\eqref{eq:wgf_velocity_general}).
\end{lemma}

\begin{proof}
Writing $\mathcal{F}(q) = \int q(x) \log\bigl(q(x)/p(x)\bigr)\,dx$ and computing the first variation,
\begin{equation}
\frac{\delta \mathcal{F}}{\delta q}(x) = \log\frac{q(x)}{p(x)} + 1,
\end{equation}
so that
\begin{equation}
\nabla_x \frac{\delta \mathcal{F}}{\delta q}(x) = \nabla_x \log q(x) - \nabla_x \log p(x).
\end{equation}
The $W_2$ gradient flow particle velocity is the steepest-descent direction $v_t(x) = -\nabla_x \frac{\delta \mathcal{F}}{\delta q_t}(x)$~\cite{Jordan1998:Variational, Ambrosio2005:GF}, which yields the score-difference form.
\end{proof}

\paragraph{Drifting field as KDE-WGF velocity on policy space (Eq.~\eqref{eq:drift_pol_wgf}).} The identification of the drifting field $\mathbf{V}_{p, q}$ with the KDE-approximated $W_2$ gradient flow velocity of $\KL(q \| p)$ is established by \cite{Cao2026:GFD}; we reproduce the specialization to policy space here for completeness. By Lemma~\ref{lem:wgf_kl_velocity} applied to $\mathcal{F}(\pi_t) = \KL(\pi_t \| \pi^+)$ with $p = \pi^+, q_t = \pi_t$, the $W_2$ gradient flow particle velocity on policy space is
\begin{equation}
v_t(a | s) = \nabla_a \log \pi^+(a | s) - \nabla_a \log \pi_t(a | s).
\label{eq:app_pol_velocity}
\end{equation}
Applying the KDE gradient identity Eq.~\eqref{eq:kde_grad} to both score functions with kernel bandwidth $h$,
\begin{equation}
h^2\, \nabla_a \log \pi^+_{\mathrm{kde}}(a | s) = \mathbf{V}^+_{\pi^+(\cdot | s)}(a), \qquad h^2\, \nabla_a \log \pi_{t, \mathrm{kde}}(a | s) = \mathbf{V}^-_{\pi_t(\cdot | s)}(a),
\end{equation}
where $\mathbf{V}^+, \mathbf{V}^-$ are the kernel mean-shift forms in Sec.~\ref{sec:prelim_drifting}. Multiplying Eq.~\eqref{eq:app_pol_velocity} by $h^2$ and substituting the KDE-approximated scores,
\begin{equation}
h^2\, v_t^{\mathrm{kde}}(a | s) = \mathbf{V}^+_{\pi^+(\cdot | s)}(a) - \mathbf{V}^-_{\pi_t(\cdot | s)}(a) = \mathbf{V}_{\pi^+, \pi_t}(a | s).
\end{equation}
Specializing $\pi_t = \pi_\theta$ gives Eq.~\eqref{eq:drift_pol_wgf}. \qed

\paragraph{Decomposition of the drifting field (Eq.~\eqref{eq:wpo_decomp}).} Taking the $a$-gradient of $\log \pi^+(a | s)$ from the closed form $\pi^+(a | s) = \pi_{\mathrm{old}}(a | s) \exp(Q_\phi(s, a)/\alpha) / Z(s)$,
\begin{equation}
\nabla_a \log \pi^+(a | s) = \nabla_a \log \pi_{\mathrm{old}}(a | s) + \frac{1}{\alpha}\, \nabla_a Q_\phi(s, a),
\end{equation}
since $\nabla_a \log Z(s) = 0$. Eq.~\eqref{eq:drift_pol_wgf} expresses $\mathbf{V}_{\pi^+, \pi_\theta}$ in terms of KDE-smoothed log densities; under the small-bandwidth limit $\log p_{\mathrm{kde}} \to \log p$, substituting the Boltzmann gradient above yields
\begin{align*}
\mathbf{V}_{\pi^+, \pi_\theta}(a | s)
&\simeq h^2\bigl[\nabla_a \log \pi^+(a | s) - \nabla_a \log \pi_\theta(a | s)\bigr] \\
&= h^2\bigl[\nabla_a \log \pi_{\mathrm{old}}(a | s) + \tfrac{1}{\alpha}\nabla_a Q_\phi(s, a) - \nabla_a \log \pi_\theta(a | s)\bigr] \\
&= \frac{h^2}{\alpha}\, \nabla_a Q_\phi(s, a) + h^2\, \bigl( \nabla_a \log \pi_{\mathrm{old}}(a | s) - \nabla_a \log\pi_\theta(a | s) \bigr),
\end{align*}
which is Eq.~\eqref{eq:wpo_decomp}. \qed

\subsection{Proof of Proposition~\ref{prop:topk_limit}}
\label{app:topk_limit_proof}

We prove the two claims of Proposition~\ref{prop:topk_limit}: (i) convergence of $\mathcal{L}_{\text{top-}K}$ to the level-set drift loss, and (ii) the total-variation bias bound to $\mathcal{L}_{\mathrm{PI}}$.

\paragraph{Setup.} Let $a^{(1)}, \ldots, a^{(N)} \overset{\text{i.i.d.}}{\sim} \pi_{\mathrm{old}}(\cdot | s)$ and let $P_K(s) := \mathrm{TopK}_{j} Q_\phi(s, a^{(j)})$ be the empirical top-$K$ set with $K/N = \rho$. Denote by $F_s$ the cumulative distribution function of $Q_\phi(s, A)$ for $A \sim \pi_{\mathrm{old}}(\cdot | s)$, and by $q^\rho(s) := F_s^{-1}(1 - \rho)$ the population $(1-\rho)$-quantile. Under the density assumption (strictly positive density of $F_s$ at $q^\rho(s)$), the empirical $(1-\rho)$-quantile $\hat{q}^\rho_N(s)$ satisfies $\hat{q}^\rho_N(s) \to q^\rho(s)$ a.s. by the Bahadur representation (or directly by Glivenko-Cantelli applied to $F_s$).

\paragraph{(i) Convergence of $\mathcal{L}_{\text{top-}K}$ to the level-set drift loss.}
The empirical distribution $P_K(s)$ is the conditional empirical measure of $\{a^{(j)}\}$ given $Q_\phi(s, a^{(j)}) \geq \hat{q}^\rho_N(s)$. By a standard truncation argument, $P_K \xrightarrow{d} \tilde{\pi}^\rho$ in $W_2$ as $N \to \infty$. The kernel mean shift $\mathbf{V}^+_p(x)$, defined by Eq.~\eqref{eq:kde_grad}, is continuous in $p$ in total variation. Writing $\mathbf{V}^+_p(x) = N(p)/D(p)$ with $N(p) := \int k(x, y)(y - x)\,dp(y)$ and $D(p) := \int k(x, y)\,dp(y)$, and using $\bigl|\int f\,(dp - dq)\bigr| \leq \|f\|_\infty\, \mathrm{TV}(p, q)$ on both with $\|k(x, \cdot)(\cdot - x)\|_\infty \leq K_{\max}\, \mathrm{diam}(\mathcal{A})$ and $\|k(x, \cdot)\|_\infty \leq K_{\max}$, the standard quotient identity yields
\begin{equation}
\bigl|\mathbf{V}^+_p(x) - \mathbf{V}^+_q(x)\bigr| \;\leq\; L_V\, \mathrm{TV}(p, q), \quad L_V \;=\; \mathcal{O}\!\left(\frac{K_{\max}^2\, \mathrm{diam}(\mathcal{A})}{k_{\min}^2}\right),
\label{eq:V_lipschitz}
\end{equation}
where $k_{\min}$ is a positive lower bound on $D(\cdot)$ (positive when $p, q$ have local full support). Combining with $P_K \to \tilde{\pi}^\rho$ in TV gives $\mathbf{V}^+_{P_K} \to \mathbf{V}^+_{\tilde{\pi}^\rho}$ pointwise. The negative side $\mathbf{V}^-_{\pi_\theta}$ is unchanged (population), so the drift field $\mathbf{V}_{P_K, \pi_\theta} \to \mathbf{V}_{\tilde{\pi}^\rho, \pi_\theta}$. The squared-norm form of $\mathcal{L}_{\mathrm{drift}}$ (Eq.~\eqref{eq:Ldrift_general}) and dominated convergence yield $\mathcal{L}_{\text{top-}K}(\theta) \to \mathcal{L}_{\mathrm{drift}}(\theta;\, \tilde{\pi}^\rho, \pi_\theta)$.

\paragraph{(ii) Bias bound to $\mathcal{L}_{\mathrm{PI}}$.}
Both $\mathcal{L}_{\mathrm{drift}}(\theta;\, \tilde{\pi}^\rho, \pi_\theta)$ and $\mathcal{L}_{\mathrm{PI}}(\theta) = \mathcal{L}_{\mathrm{drift}}(\theta;\, \pi^+, \pi_\theta)$ share the same negative side $\mathbf{V}^-_{\pi_\theta}$, so by Eq.~\eqref{eq:V_lipschitz},
\begin{equation}
\bigl\|\mathbf{V}^+_{\tilde{\pi}^\rho}(\hat{a} | s) - \mathbf{V}^+_{\pi^+}(\hat{a} | s)\bigr\| \;\leq\; L_V\, \mathrm{TV}\bigl(\tilde{\pi}^\rho(\cdot | s), \pi^+(\cdot | s)\bigr).
\end{equation}
Expanding the squared norm in Eq.~\eqref{eq:Ldrift_general} $\mathcal{L}_{\mathrm{drift}}(\theta; p, q) = \mathbb{E}\bigl[\|\mathbf{V}^+_p - \mathbf{V}^-_q\|^2\bigr]$. The difference factorizes as
\begin{equation}
\mathcal{L}_{\mathrm{drift}}(\theta;\, \tilde{\pi}^\rho, \pi_\theta) - \mathcal{L}_{\mathrm{PI}}(\theta) = \mathbb{E}\bigl[(\mathbf{V}^+_{\tilde{\pi}^\rho} - \mathbf{V}^+_{\pi^+}) \cdot (\mathbf{V}^+_{\tilde{\pi}^\rho} + \mathbf{V}^+_{\pi^+} - 2\mathbf{V}^-_{\pi_\theta})\bigr].
\end{equation}
By Cauchy-Schwarz and the bounded norm of kernel mean shift ($\|\mathbf{V}^+_p\|, \|\mathbf{V}^-_q\| \leq M$ with $M = \mathrm{diam}(\mathcal{A})$, since the mean shift is a weighted average of $(y - x)$ with $y, x \in \mathcal{A}$),
\begin{equation}
\bigl|\mathcal{L}_{\mathrm{drift}}(\theta;\, \tilde{\pi}^\rho, \pi_\theta) - \mathcal{L}_{\mathrm{PI}}(\theta)\bigr| \;\leq\; 4 M\, \mathbb{E}\bigl[\|\mathbf{V}^+_{\tilde{\pi}^\rho} - \mathbf{V}^+_{\pi^+}\|\bigr] \;\leq\; 4 M L_V\, \mathrm{TV}(\tilde{\pi}^\rho, \pi^+),
\end{equation}
which establishes the bound with constant $C := 4 M L_V = \mathcal{O}\!\left(K_{\max}^2\, \mathrm{diam}(\mathcal{A})^2 / k_{\min}^2\right)$ depending only on the kernel and action-space diameter.

\paragraph{Tightness of the bound.}
The TV gap $\overline{\mathrm{TV}}(\tilde{\pi}^\rho, \pi^+)$ is a function of both $\rho$ and $\alpha$, reflecting the structural mismatch between hard $\rho$-quantile truncation and soft Boltzmann tilting $\exp(Q_\phi/\alpha)$. In the joint sharp limit $\rho, \alpha \to 0$, both distributions collapse to $\delta_{a^\star(s)}$ at the argmax and the gap vanishes. More relevantly for our finite-parameter setting, for each soft temperature $\alpha$, there exists a matched truncation level $\rho^*(\alpha) := \arg\min_\rho \overline{\mathrm{TV}}(\tilde{\pi}^\rho, \pi^+(\alpha))$ at which the gap is minimized; equivalently, our choice of $\rho = K/N$ implicitly selects a matched $\alpha^*(\rho)$ at which the bound is tight, with the residual gap small under mild regularity of the $Q$-distribution under $\pi_{\mathrm{old}}$. \qed

\paragraph{$C$ in RL settings.}
The constant $C = \mathcal{O}\!\left(K_{\max}^2\, \mathrm{diam}(\mathcal{A})^2 / k_{\min}^2\right)$ depends on the kernel ($K_{\max}, k_{\min}$) and quadratically on the action-space diameter. Standard continuous-control RL benchmarks clip the action space to $\mathcal{A} = [-1, 1]^d$ with low dimension $d$ (e.g., Robomimic~\cite{robomimic2021} and OGBench~\cite{park2024:OGbench}), so $\mathrm{diam}(\mathcal{A}) = 2\sqrt{d}$ remains small and $C$ is moderate in practice, in contrast to the high-dimensional image-generation setting where drifting models were originally introduced~\cite{Deng2026:Drifting}.

\section{Computation Costs}

In this section, we report the computational resources used in our experiments. All experiments were conducted on Nvidia RTX 3090 GPUs with Intel Xeon Gold 6342 CPUs (96 cores). Table~\ref{tab:compute_resources} reports the GPU-hours used for each experiment.

\begin{table}[h]
\centering
\small
\caption{Computational resources for each experiment in this paper. We report the total GPU-hours aggregated across all tasks and seeds, measured on Nvidia RTX 3090 GPUs.}
\label{tab:compute_resources}
\begin{tabular}{lc}
\toprule
\textbf{Experiment} & \textbf{GPU-hours} \\
\midrule
Main offline-to-online (Table~\ref{tab:main_results})            &  3,500h  \\
Backbone $\times$ $\mathcal{L}_{\text{top-}K}$ ablation (Table~\ref{tab:topk_loss_ablation})  &  900h  \\
Top-$K$ size ablation (Table~\ref{tab:topk_ablation})             &  1,500h  \\
Drift weight $\lambda$ ablation (Table~\ref{tab:beta_ablation})   &  500h  \\
Offline RL evaluation (Sec.~\ref{app:offline_rl})                 &  100h  \\
\midrule
\textbf{Total}                                                    &  6,500h  \\
\bottomrule
\end{tabular}
\end{table}

\section{Societal Impact}
\label{app:societal_impact}
Our work introduces a non-ODE generative policy paradigm for offline-to-online reinforcement learning, with potential downstream applications in robotics, automated manufacturing, and other settings requiring learning from limited demonstrations followed by online refinement. As with reinforcement learning methods broadly, deployed policies inherit biases in the reward specification and demonstration data and may behave unpredictably under distribution shift; physical deployment requires additional safety mechanisms beyond what our simulation experiments evaluate. We do not foresee specific misuse pathways unique to this work beyond those associated with reinforcement learning for continuous control.
\clearpage

\end{document}